\documentclass[11pt]{article}

\usepackage[left=1.5cm,right=1.5cm,top=2cm,bottom=2cm]{geometry}\usepackage{authblk}

\usepackage{multicol}

\title{Passed \& Spurious: \\Descent Algorithms and Local
		Minima in Spiked Matrix-Tensor Models} 
\author[a]{Stefano Sarao Mannelli}
\author[b]{Florent Krzakala}
\author[a]{\\ Pierfrancesco Urbani}
\author[a]{Lenka Zdeborov\'a}
\affil[a]{Institut de physique th\'eorique, Universit\'e Paris Saclay, CNRS, CEA, 91191 Gif-sur-Yvette, France}
\affil[b]{Laboratoire de Physique de l’Ecole normale sup\'erieure, Universit\'e PSL, CNRS, Sorbonne Universit\'e, Universit\'e Paris-Diderot, Sorbonne Paris Cit\'e, Paris, France}
\date{}

\usepackage{microtype}
\usepackage{graphicx}
\usepackage{subfigure}
\usepackage{booktabs} 

\usepackage{amssymb}
\usepackage{mathbbol}
\usepackage{dsfont}
\usepackage{commath}

\usepackage{amsthm} 
\makeatletter
\newtheorem*{rep@theorem}{\rep@title}
\newcommand{\newreptheorem}[2]{%
\newenvironment{rep#1}[1]{%
 \def\rep@title{#2 \ref{##1}}%
 \begin{rep@theorem}}%
 {\end{rep@theorem}}}
\makeatother
\newtheorem{theorem}{Theorem} 
\newreptheorem{theorem}{Theorem} 
\newtheorem{lemma}{Lemma} 
\newreptheorem{lemma}{Lemma} 
\newtheorem{property}{Property} 
\newreptheorem{property}{Property} 
\newreptheorem{proposition}{Proposition} 
\newreptheorem{remark}{Remark} 

\newcommand{\Cmag}{m}

\allowdisplaybreaks

\usepackage{hyperref}

\usepackage{hyperref}
\hypersetup{colorlinks,linkcolor=blue, citecolor=red}

\begin{document}

\maketitle

\begin{abstract}
In this work we analyse quantitatively the interplay between the loss landscape and performance of descent algorithms in a prototypical inference problem, the spiked matrix-tensor model. We study a loss function that is the negative log-likelihood of the model. 
We analyse the number of local minima at a fixed distance from the signal/spike with the Kac-Rice formula, and locate trivialization of the landscape at large signal-to-noise ratios. 
We evaluate in a closed form the performance of a gradient flow algorithm using integro-differential PDEs as developed in physics of disordered systems for the Langevin dynamics. 
We analyze the performance of an approximate message passing algorithm estimating the maximum likelihood configuration via its state evolution. 
We conclude by comparing the above results: while we observe a
drastic slow down of the gradient flow dynamics even in the region
where the landscape is trivial, both the analyzed algorithms are shown
to perform well even in the part of the region of parameters where
spurious local minima are present. 
%
\end{abstract}

\tableofcontents

\section{Introduction}\label{sec:intro}

\begin{figure*}[ht]
	\vskip 0.2in
	\begin{center}
		\centering
		\includegraphics[width=0.49\columnwidth]{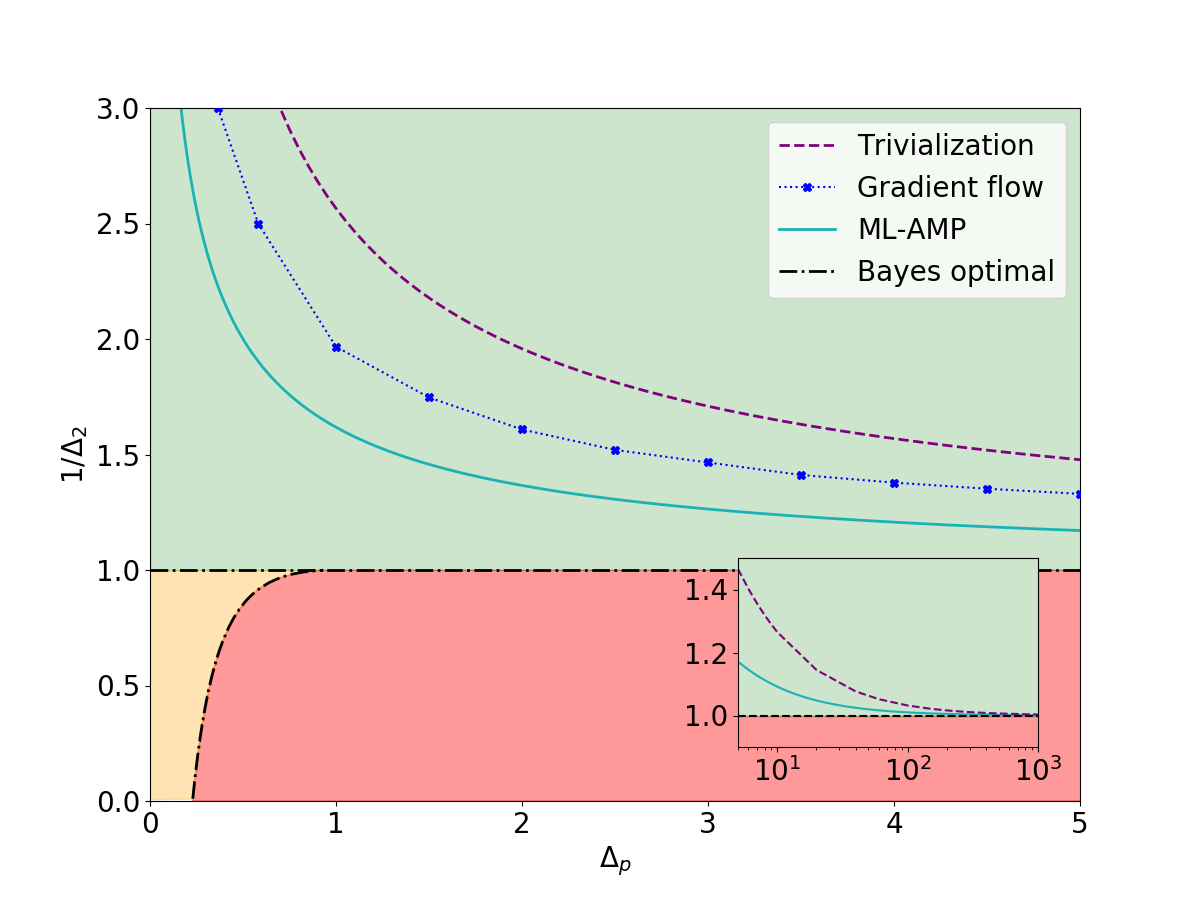}
		\includegraphics[width=0.49\columnwidth]{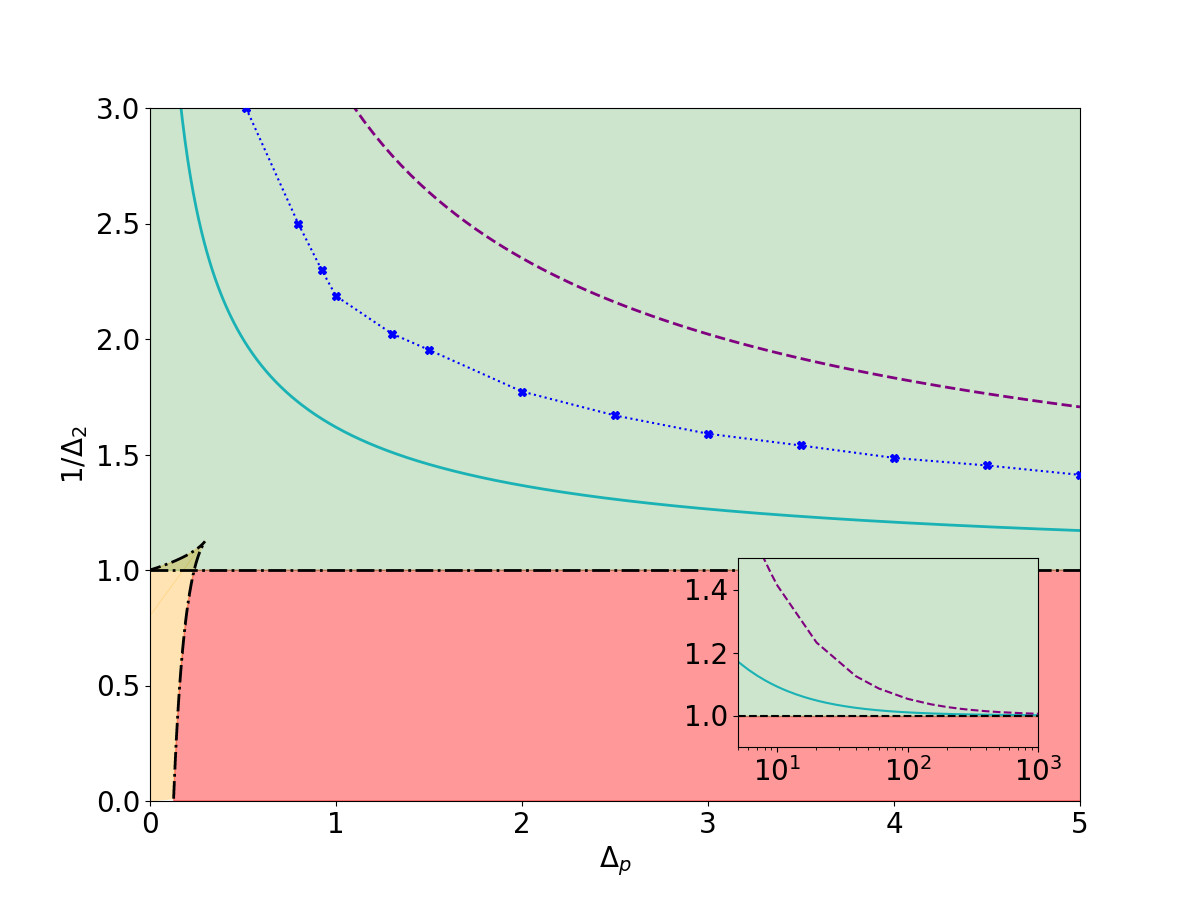}
		\caption{The figure summarizes the main results of this
		paper for the spiked matrix-tensor model with $p=3$ (left)
		and $p=4$ (right). As a function of the tensor-noise parameter
		$\Delta_p$ on the x-axes, we plot the values of $1/\Delta_2$
		above which the following happens (from above): Above
		$\Delta_2^{\rm triv}$ (the dashed purple line) the landscape of the problem becomes
		trivial in the sense that all spurious local minima
		disappear. Above $\Delta_2^{\rm GF}$ (the dotted blue line) and
		$\Delta_2^{\rm ML-AMP}$ (the full cyan line), Eq.~\eqref{eq:AMP stability m0}, the gradient flow and
		the ML-AMP
		algorithm, respectively, converge close to the ground truth signal in time
		linear in the input size. While the results for Kac-Rice and
		ML-AMP are given in a closed form, the ones for GF are obtained by
		extrapolating a convergence time obtained by numerical
		solution of integro-differential equations that describe
		large size behaviour of the GF. We note that all the three
		lines $\Delta_2^{\rm triv}$, $\Delta_2^{\rm GF}$, and
		$\Delta_2^{\rm ML-AMP}$ converge to 1 as $\Delta_p\to \infty$,
		consistently with the spiked matrix model.  These three lines,
		related to minimization of the landscape, and
		their mutual positions, are the
		main result of this paper.
		The colors in the background, separated by the black dashed-dotted lines, show
		for comparison the phase diagram for the Bayes-optimal inference, related to
		the ability to approximate the marginals of the corresponding posterior
		probability distribution, and  are taken
		from \cite{sarao2018marvels}.
		In the red region obtaining a positive correlation with the signal in
		information-theoretically impossible. In the green region it is
		possible to obtain optimal correlation with the signal using the
		Bayes-optimal AMP (BO-AMP). And in the orange the region the BO-AMP is
		not able to reach the Bayes-optimal performance.
		The insets show that in the limit
		$\Delta_p \to \infty$ all the described thresholds
		converge to the well-known BBP
		phase transition at $\Delta_2^{\rm BPP} = 1$ \cite{baik2005phase}.
		}
		%
		\label{fig:phase_diagram_simple}
	\end{center}
	\vskip -0.15in
\end{figure*}

A central question in computational sciences is the algorithmic
feasibility of optimization in high-dimensional non-convex
landscapes. This question is particularly important in
learning and inference problems where the value of the optimized
function is not the
ultimate criterium for quality of the result, instead the generalization error or the closeness to a ground-truth signal is more relevant.

Recent years brought a popular line of research into this question
where various works show for a variety of systems that spurious local
minima are not present in certain regimes of parameters and conclude
that consequently optimization algorithms shall succeed, without the
aim of being exhaustive these include \cite{kawaguchi2016deep,soudry2016no,ge2016matrix,freeman2016topology,bhojanapalli2016global,park2016non,du2017gradient,ge2017no,du2017gradient,lu2017depth}. The
	{\it spuriosity} of a minima is in some works defined by their distance from the
global minimum, in other works as local minimizers that lead
to bad generalization or bad accuracy in reconstruction of the ground
truth signal. These two notions are not always equivalent, and
certainly the later is more relevant and will be used in the
present work.

Many of the existing works stop at the statement that absence of
spurious local minimizers leads to algorithmic feasibility and the
presence of such spurious local minima leads to algorithmic difficulty, at least as far as gradient-descent-based algorithms are concerned. At the same
time, even gradient-descent-based algorithms may be able to perform well even when spurious local
minima are present. This is because the basins of attraction of the
spurious minimas may be small and the dynamics might be able to avoid
them. In the other direction, even if spurious local minima are
absent, algorithms might take long time to find a minimizer for
entropic reasons that in high-dimensional problems may play a crucial
role.

{\bf Main Results:}
We study the spiked matrix-tensor model,
introduced and motivated in \cite{sarao2018marvels}. We view this
model as a prototypical solvable example of a high-dimensional non-convex
optimization problem, and anticipate that the results observed here
will have a broader relevance. Our main contributions are:
\begin{itemize}
	\item  Using the Kac-Rice formula
	      \cite{fyodorov2004complexity,arous2017landscape} we
	      rigorously derive the expected number of local minimizers of the associated likelihood at a given
	      correlation with the ground truth signal.
	\item We extend the recently introduced {\it
			      Langevin-state-evolution} \cite{sarao2018marvels} to the
	      gradient flow (GF) algorithm, obtaining a closed-form
	      formula for the obtained accuracy in the limit of large
	      system sizes. This formula is conjectured exact, and could
	      likely be established by extending \cite{arous2006cugliandolo}.
	\item We derive and analyze the state evolution that
	      rigorously describes the performance of the maximum-likelihood version of the approximate message
	      passing algorithm (ML-AMP) for the present model.
\end{itemize}

We show that the above two algorithms (GF and ML-AMP) achieve the same error
in the regime where they succeed. That same value of the error is also deduced from the
position of all the minima strongly correlated with the signal as
obtained from the Kac-Rice approach (precise statement below). We
quantify the region of parameters in which the two above algorithms
succeed and show that the ML-AMP is strictly better than GF.
Remarkably, we show that the algorithmic performance is not driven by the
absence of spurious local minima. These results are summarized
in Fig.~\ref{fig:phase_diagram_simple} and show that, in
order to obtain a complete picture for settings beyond the present
model, the precise interplay between absence of spurious local minima and
algorithmic performance remains to be further investigated.


\section{Problem Definition}\label{sec:model}

In this paper we consider the spiked matrix-tensor model as studied in
\cite{sarao2018marvels}. This is a statistical inference problem
where the ground truth signal $x^*\in\mathbb{R}^N$ is sampled
uniformly on the $N-1$-dimensional sphere, $\mathbb{S}^{N-1}(\sqrt
	N)$. We then obtain two types of observations about the signal, a
symmetric matrix
$Y$, and an order $p$ symmetric tensor $T$, that given the signal $x^*$ are
obtained as
\begin{align}\label{eq:channel Y}
	 & Y_{ij} = \frac{x_i^{*}x_j^*}{\sqrt{N}}+\xi_{ij},                                                  \\
	\label{eq:channel T}
	 & T_{i_1,\dots,i_p} = \frac{\sqrt{(p-1)!}}{N^{(p-1)/2}}x_{i_1}^*\dots x_{i_p}^*+\xi_{i_1,\dots,i_p}
\end{align}
for $1 \le i < j \le N$ and $1 \le i_1 < \dots < i_p \le N$, using the symmetries to obtain
the other non-diagonal components. Here $\xi_{ij}$ and
$\xi_{i_1,\dots,i_p}$ are for each $i<j$ and each $i_1 < \dots <  i_p$
independent Gaussian random numbers of zero mean and variance $\Delta_2$
and $\Delta_p$, respectively.

The goal in this spiked matrix-tensor inference problem is to estimate
the signal $x^*$ from the knowledge of the matrix $Y$ and tensor
$T$. If only the matrix was present, this model reduces to well known
model of low-rank perturbation of a random symmetric matrix,  closely
related to the spiked
covariance model \cite{johnstone2001distribution}. If on the contrary only the tensor is
observed then the above model reduces to the spiked tensor model as
introduced in \cite{richard2014statistical} and studied in a range of
subsequent papers.

In this paper we study the matrix-tensor model where the two
observations are combined. Our motivation is
similar to the one exposed in \cite{sarao2018marvels}, that is, we aim to
access a regime in which it is algorithmically tractable to obtain
good performance with corresponding message passing
algorithms yet it is challenging (e.g. leading to non-convex optimization) with sampling or gradient descent based
algorithms, this happens when both $\Delta_2 = \Theta(1)$ and $\Delta_p =
	\Theta(1)$, while $N\to \infty$ \cite{sarao2018marvels}.

In this paper we focus on algorithms that aim to find the maximum
likelihood estimator. The negative log-likelihood (Hamiltonian in physics, or loss function in machine learning) of the spiked matrix-tensor
reads
\begin{equation}\label{eq:Hamiltonian}
	\begin{split}
		\mathcal{L} &= \sum_{i<j}\frac1{2\Delta_2}\left(Y_{ij}-\frac{x_ix_j}{\sqrt{N}}\right)^2+\sum_{ i_1<\dots<i_p}\frac1{2\Delta_p}\left(T_{i_1\dots i_p}-\frac{\sqrt{(p-1)!}}{N^{(p-1)/2}}x_{i_1}\dots x_{i_p}\right)^2,
	\end{split}
\end{equation}
where $x \in \mathbb{S}^{N-1}(\sqrt
	N)$ is constrained to the sphere.

In a high-dimensional, $N\to
	\infty$, noisy regime the maximum-likelihood estimator is not always
optimal as it provides in general larger error than the Bayes-optimal
estimator computing the marginals of the posterior, studied in
\cite{sarao2018marvels}. At the same time the
log-likelihood (\ref{eq:Hamiltonian}) can be seen as a loss function,
that is non-convex and high-dimensional. The tractability and
properties of such minimization problems are the most questioned
in machine learning these days, and are worth detailed investigation
in the present model.


\section{Landscape Characterization}\label{sec:kac-rice}

\begin{figure*}
	\vskip 0.2in
	\begin{center}
		\centering
		\subfigure[$\Delta_2 = 2$]{
		\includegraphics[width=.31\columnwidth]{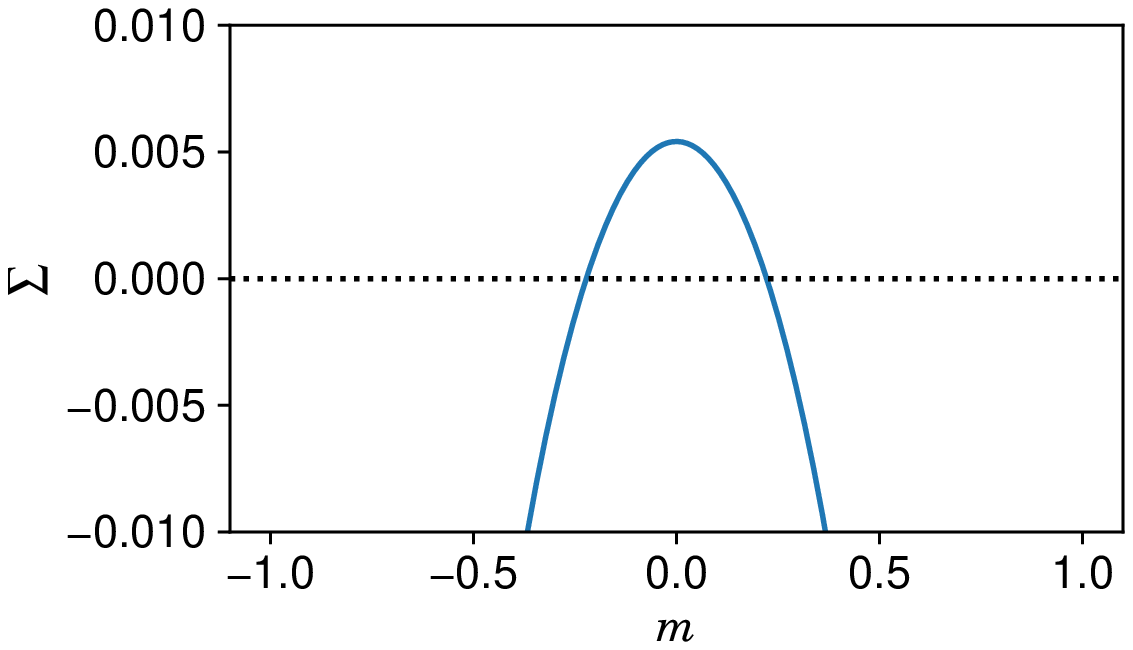}
		}
		\subfigure[$\Delta_2 = 2/3$]{
		\includegraphics[width=.31\columnwidth]{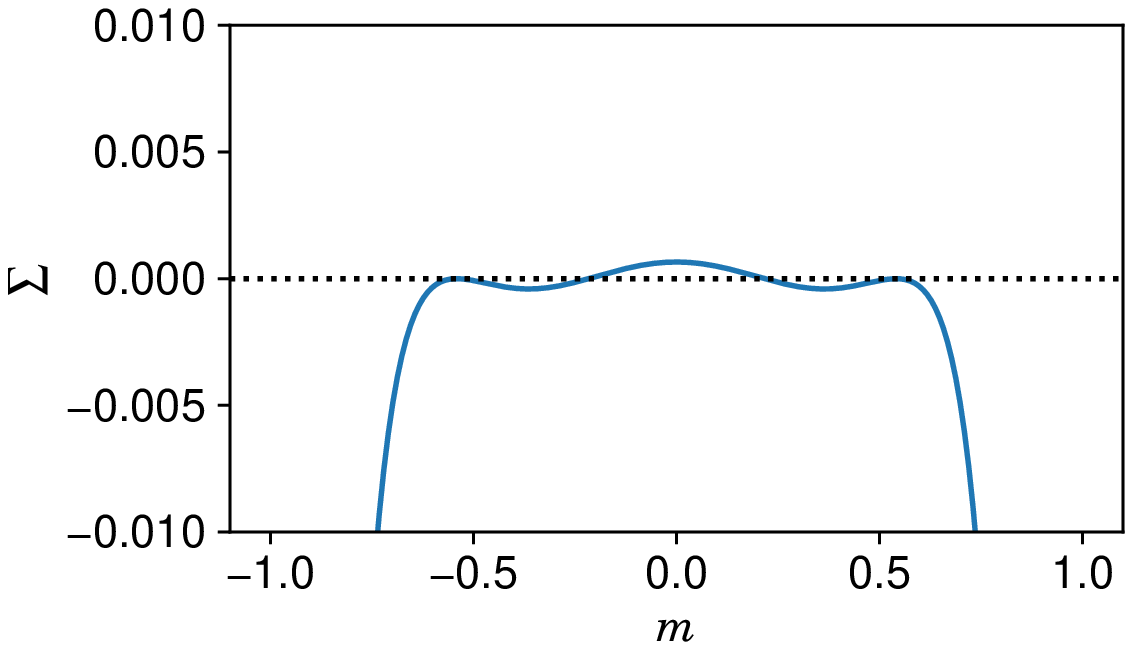}
		}
		\subfigure[$\Delta_2 = 2/5$]{
		\includegraphics[width=.31\columnwidth]{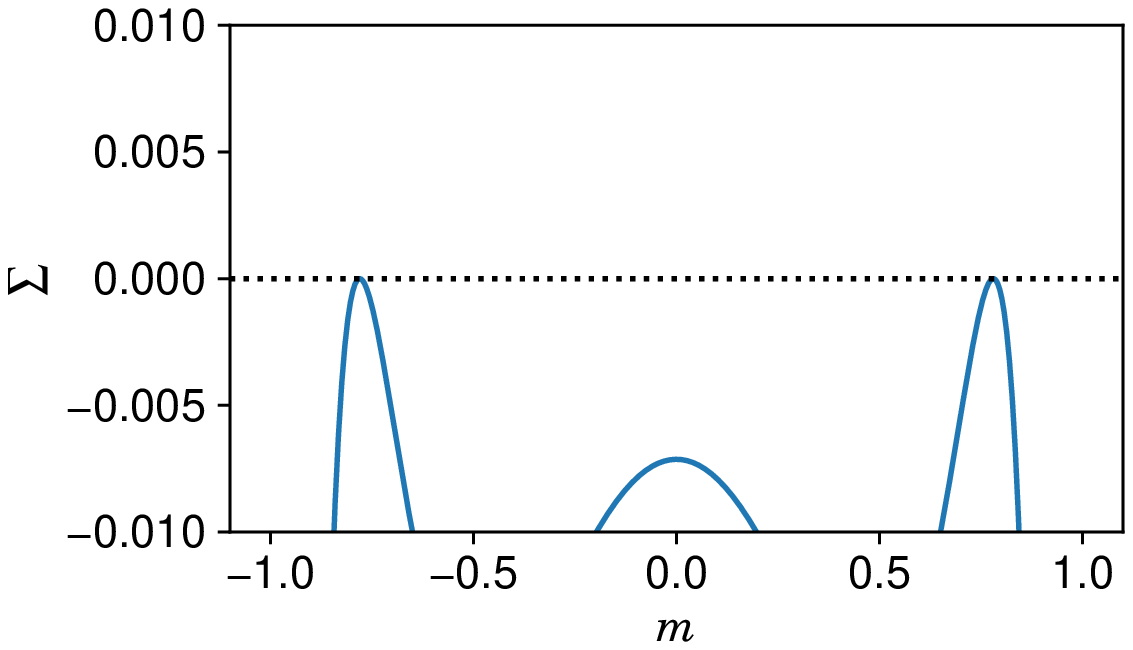}
		}
		\caption{The complexity $\Sigma(m)$, Eq.~\eqref{eq:complexity}, is shown for different
			values of parameter $\Delta_2$ at fixed $\Delta_p=4.0$ in
			the case $p=4$. As $\Delta_2$ is decreased (the signal
			to noise ratio increases) the complexity allows to identify
			three main scenarios in the topology of the loss landscape. In
			the first case (a) only a wide band of non-negative complexity
			around the point of zero correlation is present, in the
			second case (b) minima with non-trivial correlation with
			the signal appear but the band around $m=0$ is still present,
			finally (c) the signal dominates over the noise and only minima
			with non-trivial correlation are present. The transition
			from case (b) to case (c), i.e. when the support of
			$\Sigma(m) \ge 0$ becomes two discontinuous points, as the
			bulk close to $m=0$ becomes negative, is called the landscape trivialization. The
			$\Delta_2$ at which this occurs is denoted $\Delta_2^{\rm
					triv}$ and depicted in dashed purple in Fig.~\ref{fig:phase_diagram_simple}.
		}
		\label{fig:complexity}
	\end{center}
	\vskip -0.2in
\end{figure*}

The first goal of this paper is to characterize the structure of local minima of
the loss function (equivalently local maxima of the log-likelihood)
eq.~(\ref{eq:Hamiltonian}) as a function of the noise parameters
$\Delta_2$ and $\Delta_p$. We compute the average
number of local minimizers $x$ having a given correlation with the ground
truth signal  $m = \lim_{N\to \infty} x \cdot x^*/N$. This leads to a so-called {\it complexity} function
$\Sigma(m)$ defined as the logarithm of the
expected number of local minima at correlation $m$ with the ground
truth.

A typical example of this function, resulting from our analysis, is depicted in
Fig.~\ref{fig:complexity} for $p=4$, $\Delta_p=4.0$, and several values of
$\Delta_2$. We see from the figure that at large $\Delta_2$ local minima
appear only in a narrow range of values of $m$ close to zero, as
$\Delta_2$ decreases the
support of $\Sigma(m) \ge 0$ widens. At yet smaller values of
$\Delta_2$ the support $\Sigma(m) \ge 0$ becomes disconnected so that
it is supported on an interval of value close to $m=0$ and on two (one
negative, one positive) isolated points.
For yet smaller $\Delta_2$ the complexity for values of $m$ close to zero becomes
negative, signalling what we call a {\it trivialization of the landscape}, where all
remaining local minima are (in the leading order in $N$) as correlated
with the ground truth as the global minima. The support of $\Sigma(m)
	\ge 0$ in the trivialized regions consists of two separated points. We call the value of
$\Delta_2$ at which the trivialization happens $\Delta_2^{\rm triv}$. In the phase diagram of
Fig.~\ref{fig:phase_diagram_simple} the trivialization of the energy landscape happens above
the purple dashed line.

We use the Kac-Rice formula to determine the complexity
$\Sigma(m)$ \cite{adler2009random,fyodorov2013high}. Given an
arbitrary continuous function, the Kac counting formula allows to
compute the number of points where the function crosses a given
value. The number
of minima can be characterized using Kac's formula on the gradient of
the loss (\ref{eq:Hamiltonian}), counting how many time the
gradient crosses the zero value, under the condition of having a
positive definite Hessian in order to count only local minima and not
saddles. Since the spiked matrix-tensor model is characterized by a
random landscape, due to the noise $\xi_{ij}$ and
$\xi_{i_1,\dots,i_p}$, we
will consider the expected number of minima obtaining the Kac-Rice
formula \cite{adler2009random,fyodorov2013high}.

For mathematical convenience we will consider the rescaled
configurations $\sigma = x/\sqrt{N} \in\mathbb{S}^{N-1}(1)$, and
rescaled signal $\sigma^* = x^*/\sqrt{N}$.
Call $\phi_{G,F_2,F_p}$ the joint probability density of the
gradient $G$ of the loss, and of the $F_2$ and $F_p$ the contributions of the matrix and tensor
to the loss, respectively. Given the value of the two
contributions to the loss  $F_2=\epsilon_2 N$ and $F_p=\epsilon_p N$, and the correlation between the configuration and ground truth $m\in [-1,+1]$ that we impose using a Dirac's delta, the averaged number of minimizers is
\begin{equation}
	\begin{split}
		\mathcal{N}&(m,\epsilon_2,\epsilon_p;\Delta_2,\Delta_p) = e^{\tilde{\Sigma}_{\Delta_2,\Delta_p}(m,\epsilon_2,\epsilon_p)} =
		\\
		& =\int_{\mathbb{S}^{N-1}} \! \! \! \! \! \!  \mathbb{E}\left[\det H\big|G=0,\frac{F_2}{N}=\epsilon_2,\frac{F_p}{N}=\epsilon_p,H\succ0\right]		
		\phi_{G,F_2,F_p}(\sigma,0,\epsilon_2,\epsilon_p)\,
		\delta(m-\sigma\cdot\sigma^*) \, {\rm d}\sigma\,.
	\end{split}
\end{equation}
Rewrite the loss Eq.~\eqref{eq:Hamiltonian} neglecting terms that are constant with respect to the configuration and thus do not contribute to the complexity
\begin{equation}\label{eq:HamiltonianKR}
	\begin{split}
		\hat{\mathcal{L}} &= \frac{\sqrt{N(p-1)!}}{\Delta_p}\sum_{i_1<\dots<i_p} \xi_{i_1\dots i_p}\sigma_{i_1}\dots \sigma_{i_p}
		- \frac{N(p-1)!}{\Delta_p}\sum_{i_1<\dots<i_p}\sigma^*_{i_1}\sigma_{i_1}\dots \sigma^*_{i_p}\sigma_{i_p}
		\\
		&+
		\frac{\sqrt{N}}{\Delta_2}\sum_{i<j}\xi_{ij}\sigma_i\sigma_j
		-
		\frac{N}{\Delta_2}\sum_{i<j}\sigma^*_i\sigma_i\sigma^*_j\sigma_j\, .
	\end{split}
\end{equation}
In the following we will use small letters $f_2$, $f_p$, $g$, $h$ to characterize losses, gradient and Hessian constrained on the sphere and capital letters for the same quantities unconstrained. Define $\mathbb{I}_d$ the $d$-dimensional identity matrix. The following lemma characterizes $\phi_{G,F_2,F_p}$.
\begin{lemma}
	Given the loss function Eq.~\eqref{eq:HamiltonianKR} and a configuration $x$ such that the correlation and the signal is $m$, then there exists a reference frame such that the joint probability distribution of $f_2, f_p \in \mathbb{R}$, $g\in\mathbb{R}^{N-1}$ and $h \in \mathbb{R}^{(N-1)\times(N-1)}$ is given by
	\begin{align}
		\label{eq:KR_f_rv}
		 & \frac{f_k}{N} \sim -\frac1{k\Delta_k}m^k + \frac1{\sqrt{k\Delta_k}}\frac1{\sqrt{N}}{Z}_k\,;
		\\
		\label{eq:KR_g_rv}
		\begin{split}
			& \frac{g}{N} \sim -\left(\frac1{\Delta_p}m^{p-1} + \frac1{\Delta_2}m\right)\sqrt{1-m^2} \mathbb{e}_1
			+ \sqrt{\frac1{\Delta_p}+\frac1{\Delta_2}}\frac1{\sqrt{N}}\tilde{\textbf{Z}}\,;
		\end{split}
		\\
		\label{eq:KR_h_rv}
		\begin{split}
			& \frac{h}{N} \sim -\left(\frac{p-1}{\Delta_p}m^{p-2} + \frac{1}{\Delta_2}\right)(1-m^2) \mathbb{e}_1\mathbb{e}_1^T
			+ \sqrt{\frac{p-1}{\Delta_p}+\frac{1}{\Delta_2}}\sqrt{\frac{N-1}{N}}\mathbb{W} - (pf_p + 2f_2)\mathbb{I}_{N-1}\,;
		\end{split}
	\end{align}
	with $Z_k$ standard Gaussians and $k\in\{2,p\}$, $\tilde{\textbf{Z}}\sim\mathcal{N}(0,\mathbb{I}_{N-1})$ a standard multivariate Gaussian and $\mathbb{W}\sim\text{GOE}(N-1)$ a random matrix from the Gaussian orthogonal ensemble.
\end{lemma}
\begin{proof}
	Starting from Eq.~\eqref{eq:HamiltonianKR}, split the
	contributions of the matrix and tensor in $F_2$ and $F_p$, two Gaussian variables and impose the spherical constrain with a Lagrange multiplier $\mu$.
	\begin{align}
		\label{eq:KR_f_sphere}
		\begin{split}
			& f_2(\sigma)+f_p(\sigma) = F_2(\sigma) + F_p(\sigma) - \frac\mu2 \big(\sum_i \sigma_i^2-1\big)\,,
		\end{split}
		\\
		\label{eq:KR_g_sphere}
		 & g_i(\sigma) = G_i(\sigma) - \mu \sigma_i\,, \\
		\label{eq:KR_h_sphere}
		 & h_{ij}(\sigma) = H_{ij}(\sigma) -\mu\,.
	\end{align}
	The expression for $\mu$ in a critical point can be derived as
	follows. Given $g_i(\sigma)\equiv0$, multiply
	Eq.~\eqref{eq:KR_g_sphere} by $\sigma_i$, sum over the indices
	and obtain: $\mu = \sum_i G_i(\sigma)\sigma_i =
		2f_2(\sigma) + pf_p(\sigma)$. We now restrict our study to the unconstrained random variables and substitute $\mu$. Since the quantities $f_2$, $f_p$, $g$, $h$, $\mu$ are linear functionals of Gaussians they will be distributed as Gaussian random variables and therefore can be characterized by computing expected values and covariances.
	Starting from the losses coming from the matrix and the
	tensor in Eq.~\eqref{eq:HamiltonianKR}, $F_2(\sigma)$ and
	$F_p(\sigma)$, respectively, consider the moments with respect to the realization of the noise, $\xi_{i_1\dots i_p}$, $\xi_{ij}$. For $k\in\{2,p\}$ the first moment leads to
	\begin{equation}
		\mathbb{E}[F_k(\sigma)] = -\frac{N}{k\Delta_k}(\sigma \cdot \sigma^*)^{k} + O(1)\,.
	\end{equation}
	Let's consider the second moment but having two different configurations $\sigma$ and $\tau$,
	\begin{equation}
		\mathbb{E}\left[F_k(\sigma)F_k(\tau)\right] = \frac{N}{k\Delta_k}( \sigma \cdot\tau)^k + O(1)\,.
	\end{equation}
	Using standard results for derivatives of Gaussians (see e.g. \cite{adler2009random} Eq.~5.5.4) we can obtain means and covariances of the random variables taking derivatives with respect to $\sigma$ and $\tau$. Then set $\tau = \sigma$, imposing the spherical constrain and using $\sigma\cdot\sigma^*= m$.

	The last step is the definition of a convenient reference frame $\{\mathbb{e}_j\}_{j=1,\dots,N}$. Align the configuration along the last coordinate $\mathbb{e}_N = \sigma$ and the signal with a combination of the first and last coordinates $\sigma^* = \sqrt{1-m^2}\mathbb{e}_1 + m\mathbb{e}_N$. Finally, project on the sphere by discarding the last coordinate.
\end{proof}

\vskip -0.1in

We can now rewrite the determinant of the conditioned Hessian by grouping the multiplicative factor in front of the GOE in Eq.~\eqref{eq:KR_h_rv}
\begin{equation}
	\begin{split}
		\det h &= \left(\frac{p-1}{\Delta_p}+\frac{1}{\Delta_2}\right)^{\frac{N-1}2}\left(\frac{N}{N-1}\right)^{-\frac{N-1}2}
		\det\left[\mathbb{W} + t_N \mathbb{I}_{N-1} - \theta_N \mathbb{e}_1\mathbb{e}_1^T\right]
	\end{split}
\end{equation}
with $t_N$ and $\theta_N$ given by
\begin{align}
	 & t_N\rightarrow t = -\frac{p\epsilon_p+2\epsilon_2}{\sqrt{\frac{p-1}{\Delta_p}+\frac{1}{\Delta_2}}},                                       \\
	 & \theta_N\rightarrow\theta = \frac{\frac{p-1}{\Delta_p} m^{p-2}+\frac{1}{\Delta_2}}{\sqrt{\frac{p-1}{\Delta_p}+\frac{1}{\Delta_2}}}(1-m^2)
\end{align}
in the large $N$-limit.
Therefore the Hessian behaves like a GOE shifted by $t$ with a rank
one perturbation of strength $\theta$. This exact same problem has
already been studied in \cite{arous2017landscape} and we can thus
deduce the expression for the complexity as
\begin{equation}\label{eq:complexity_annealed}
	\begin{split}
		\tilde{\Sigma}_{\Delta_2,\Delta_p}(m,\epsilon_2,\epsilon_p) &= \frac12\log\frac{\frac{p-1}{\Delta_p}+\frac{1}{\Delta_2}}{\frac1{\Delta_p}+\frac1{\Delta_2}} + \frac12\log(1-m^2)
		-\frac12\frac{\left(\frac{m^{p-1}}{\Delta_p}+\frac{m}{\Delta_2}\right)^2}{\frac1{\Delta_p}+\frac1{\Delta_2}}(1-m^2) +
		\\
		&- \frac{p\Delta_p}2\left(\epsilon_p+\frac{m^p}{p\Delta_p}\right)^2
		- \Delta_2\left(\epsilon_2+\frac{m^2}{2\Delta_2}\right)^2 + \Phi(t) - L(\theta,t),
	\end{split}
\end{equation}
with
\begin{equation*}
	\Phi(t) = \frac{t^2}4 + \mathbb{1}_{|t|>2}\left[\log\left(\sqrt{\frac{t^2}4-1}+\frac{|t|}2\right)-\frac{|t|}4\sqrt{t^2-4}\right]
\end{equation*}
\begin{equation*}
	L(\theta,t) = \begin{cases}
		 &
		\begin{split}
			& \frac14\int_{\theta+\frac1\theta}^t \sqrt{y^2-4}dy -\frac\theta2\left(t-\left(\theta+\frac1\theta\right)\right)
			\\
			& +\frac{t^2-\left(\theta+\frac1\theta\right)^2}8 \quad \theta>1,\; 2\le t<\frac{\theta^2+1}\theta
		\end{split}
		\\
		 & \infty \quad\quad t<2          \\
		 & 0 \quad\quad \text{otherwise.} \\
	\end{cases}
\end{equation*}
We note at this point that for the case of the pure spiked tensor
model $\Delta_2\to \infty$ the above expression reduces exactly to the
complexity derived in \cite{arous2017landscape}.
The following theorem states that to the leading order Eq.~\eqref{eq:complexity_annealed} represents the complexity of our problem.
\begin{theorem}
	Given $\Delta_2$ and $\Delta_p$, for any $(\epsilon_2,\epsilon_p) \in\mathbb{R}^2$ and $m\in[-1,+1]$ it holds
	\begin{equation}
		\begin{split}
			\lim_{N\rightarrow\infty}\frac1N\log\mathbb{E}\mathcal{N}&(m,\epsilon_2,\epsilon_p;\Delta_2,\Delta_p) =
			 \tilde{\Sigma}_{\Delta_2,\Delta_p}(m,\epsilon_p,\epsilon_2)
		\end{split}
	\end{equation}
\end{theorem}
\begin{proof}
	The proof comes immediately from \cite{arous2017landscape} Thm. 2, see also Sec. 4.1.
\end{proof}
The quantity that we are interested in is the projection of Eq.~\eqref{eq:complexity_annealed} to the maximizing values of $\epsilon_2$ and $\epsilon_p$:
\begin{equation}\label{eq:complexity}
	\Sigma(m) = \max_{\epsilon_2,\,\epsilon_p} \tilde{\Sigma}_{\Delta_2,\Delta_p}(m,\epsilon_2,\epsilon_p).
\end{equation}
Eq.~\eqref{eq:complexity} allows to understand if at a given
correlation with the signal, there are regions with an exponential
expected number of minima, see Fig.~\ref{fig:complexity}. Thus it
allows to locate parameters where the landscapes is trivial.

We computed the expected number of minima, i.e. the so-called
annealed average. The annealed average might be dominated by rare
samples, and in general provides only an upper bound for typical
samples. The quenched complexity, i.e. the average of the logarithm of
the number of minima, is more involved. The quenched calculation was
done in the case of a the spiked tensor model \cite{ros2018complex}.
It is interesting to notice that in \cite{ros2018complex} the authors
found that the annealed complexity does not differ from the quenched
complexity for $m=0$. This combined with analogous preliminary results
for the spiked matrix-tensor model, suggest that considering the
quenched complexity would not change the conclusions of this paper
presented in the phase diagrams Fig.~\ref{fig:phase_diagram_simple}.

\section{Gradient Flow Analysis}\label{sec:GD}

In this section we analyze the performance of the gradient flow
descent in the loss function (\ref{eq:Hamiltonian})
\begin{equation}
	\frac{d}{dt}x_i(t) = -\mu(t)x_i(t) - \frac{\delta
		\mathcal{L}}{\delta x_i}(t)\,, \label{eq:gradient}
\end{equation}
where the Lagrange parameter $\mu(t)$ is set in a way to ensure the spherical
constraint $x \in \mathbb{S}^{N-1}(\sqrt N)$. Our aim is to understand
the final correlation between the ground truth signal and the
configuration reached by the gradient flow in large but finite time,
while $N\to \infty$.

The gradient flow (\ref{eq:gradient}) can be seen as a
zero-temperature limit of the Langevin algorithm where
\begin{equation}
	\frac{d}{dt}x_i(t) = -\mu(t)x_i(t) - \frac{\delta \mathcal{L}}{\delta x_i}(t) -\eta_i(t)\,,
\end{equation}
with $\eta_i(t)$ being the Langevin noise with zero mean and
covariance $\left\langle\eta_i(t)\eta_j(t')\right\rangle =
	2T\delta_{ij}\delta(t-t')$, where $T$ has the physical meaning of
temperature, the notation $\langle\dots \rangle$ stands for the
average over the noises $\xi_{ij}$ and $\xi_{i_1,\dots,i_p}$. As we take the limit $T\rightarrow0$, the noise becomes
peaked around zero, effectively recovering the gradient flow.

The performance of the Langevin algorithm was characterized recently
in \cite{sarao2018marvels} using equations developed in physics of
disordered systems
\cite{CHS93,cugliandolo1993analytical}. In
\cite{sarao2018marvels}  this characterization
was given for an arbitrary temperature $T$ and compared to the
landscape of the Bayes-optimal estimator \cite{antenucci2018glassy}. Here
we hence summarize and use the results of \cite{sarao2018marvels}
corresponding to the limit $T\rightarrow0$.

The Langevin dynamics with generic temperature is in the large size
limit, $N\to \infty$, characterized by a set of PDEs for the
self-correlation $C(t,t') = \lim_{N\to \infty}\left\langle\frac1N\sum
	x_i(t)x_i(t')\right\rangle$, the response function $R(t,t') =
	\lim_{N\to \infty} \left\langle\frac1N\sum\frac{\delta x_i(t)}{\delta
		\eta_i(t')}\right\rangle$, and the correlation with the signal
$m(t) = \lim_{N\to \infty} \left\langle\frac1N\sum
	x_i(t)x_i^*\right\rangle$. Ref. \cite{sarao2018marvels} established
that as the gradient flow evolves these quantities satisfy
eqs.~(74)-(76) in that paper. Taking the zero-temperature limit in those equations we obtain
\begin{align}
	\begin{split}
		&\frac{\partial}{\partial t} C(t,t') =-\tilde{\mu}(t)C(t,t')+
		Q'(\Cmag(t)) \Cmag(t')
		+ \int_0^t dt''
		R(t,t'')Q''(C(t,t''))C(t',t'')
		\\
		&\quad+ \int_0^{t'}dt'' R(t',t'')Q'(C(t,t''))\,,
	\end{split}\label{eq:GDSE C}
	\\
	\begin{split}
		&\frac{\partial}{\partial t} R(t,t')=-\tilde{\mu}(t)R(t,t')
		+\int_{t'}^tdt''R(t,t'')Q''(C(t,t''))R(t'',t')\,,
	\end{split}\label{eq:GDSE R}
	\\
	\begin{split}
		&\frac{\partial}{\partial t}  \Cmag(t) =-\tilde{\mu}(t)\Cmag(t)+Q'(\Cmag(t))
		+ \int_{0}^tdt'' R(t,t'')\Cmag(t'') Q''(C(t,t''))\,,
	\end{split}\label{eq:GDSE Cbar}
\end{align}
with $Q(f) = f^p/(p\Delta_p) + f^2/(2\Delta_2)$ and
$\tilde{\mu}(t)=\lim_{T\rightarrow0}T \mu(t)$ the rescaled spherical
constraint. Boundary conditions for the equations are
$C(t,t)=1\;\forall t$, $R(t,t')=0$  for all $t<t'$ and
$\lim_{t'\rightarrow t^-}R(t,t')=1\;\forall t$. An additional equation
for $\tilde{\mu}(t)$ is obtained by fixing $C(t,t)=1$ in
Eq.~\eqref{eq:GDSE C}. In the context of disordered systems those
equations have been established rigorously for a related case of the
matrix-tensor model without the spike \cite{arous2006cugliandolo}.

\vskip -0.2in

%
\begin{figure}[ht]
	\vskip 0.2in
	\begin{center}
		\centering
		\includegraphics[width=.5\columnwidth]{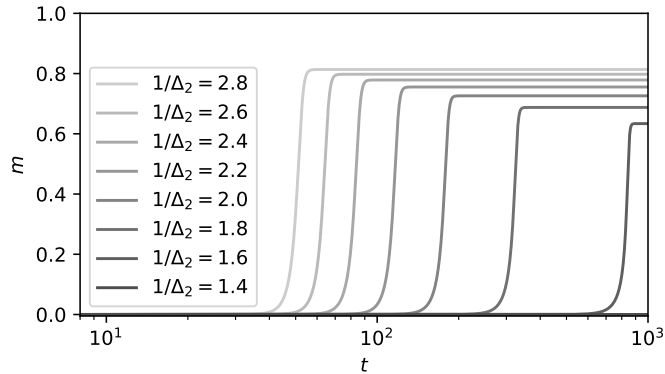}
		\caption{Eq.~\eqref{eq:GDSE Cbar} characterizes the evolution
		of the correlation of the gradient flow with the ground
		truth signal, evaluated for several values of $\Delta_2$, at
		$\Delta_p = 4.0$ starting from $\Cmag(0)=10^{-10}$. The
		dynamics displays a fast increase of the convergence time as
		$\Delta_2$ increases. At large times, the plateau we
		observe has the same value of correlation $m$ as the minima
		best correlated with the signal, as predicted via Kac-Rice
		approach.}
		\label{fig:GD_evolution}
	\end{center}

\end{figure}
%


Eqs.~(\ref{eq:GDSE C}-\ref{eq:GDSE Cbar}) are integrated numerically
showing the large-size-limit performance of the gradient flow
algorithm. Example of this evolution is given in
Fig.~\ref{fig:GD_evolution} for $p=3$, $\Delta_p=4$. The code is available online \cite{LSEcode_T0} and linked to this paper. For consistency we confirm
numerically that at large times the gradient flow reaches
values of the correlation that correspond exactly to the value of the
correlation of the minima correlated to the signal as obtained in the
Kac-Rice approach.

As the variance $\Delta_2$ increases the time it takes to the gradient flow to acquire good
correlation with the signal increases. We define the {\it convergence}
time $t_c$ as the time it takes to reach 1/2 of the final plateau. The
dependence of $t_c$ on $\Delta_2$ is consistent with a power law
divergence at $\Delta_2^{\rm GF}$. This is illustrated in
Fig.~\ref{fig:GD_limits} where we plot the convergence time as a
function of $\Delta_2$ and show the power-law fit in the inset. The
points $\Delta_2^{\rm GF}$ are collected and plotted in
Fig.~\ref{fig:phase_diagram_simple}, dotted blue
line.

\begin{figure}
	\vskip 0.2in
	\begin{center}
		\centering
		\includegraphics[width=0.5\columnwidth]{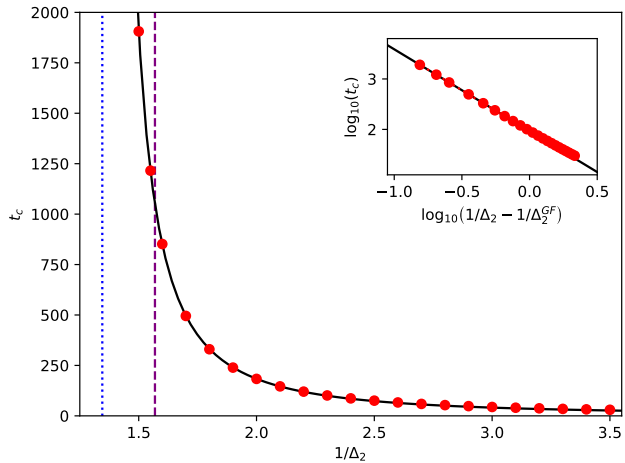}
		\caption{
		The convergence time the gradient flow takes to find a
		configuration well correlated with the signal for
		$\Delta_p=4.0$, $p=3$ as a function of $\Delta_2$, starting
		from $\Cmag(0)=10^{-10}$. The points are fitted
		with a power law consistent with a divergence point
		$1/\Delta_2^{\rm GF}
			=  1.35$ (vertical dotted line, log-log scale of the fit
		shown in the inset) while landscape trivialization occurs at
		$1/\Delta_2^{\rm triv} = 1.57$ (vertical dashed line).
		{\color{white} aaaaaaaaaaaaaaaaaaaaaaaaaaaaaaaaaaaaaaaaaaaaaaaaaaaaaaaaaaaaaaaaaaaaaaaaaaaaaaa }
		{\color{white} aaaaaaaaaaaaaaaaaaaaaaaaaaaaaaaaaaaaaaaaaaaaaaaaaaaaaaaaaaaaaaaaaaaaaaaaaaaaaaa }
		{\color{white} aaaaaaaaaaaaaaaaaaaaaaaaaaaaaaaaaaaaaaaaaaaaaaaaaaaaaaaaaaaaaaaaaaaaaaaaaaaaaaa }
		}
		\label{fig:GD_limits}
	\end{center}
	\vskip -0.6in
\end{figure}

From Fig.~\ref{fig:GD_limits} we see that the gradient flow algorithm
undergoes a considerable slow-down even in the region where the
landscape is trivial, i.e. does not have spurious local minimizers. At
the same time divergence of the convergence time happens only well inside the phase where spurious local
minimizers do exist.

\section{Maximum-Likelihood Approximate Message Passing}\label{sec:0T-AMP}

\begin{figure}[ht]
	\vskip 0.2in
	\begin{center}
		\centering
		\includegraphics[width=0.5\columnwidth]{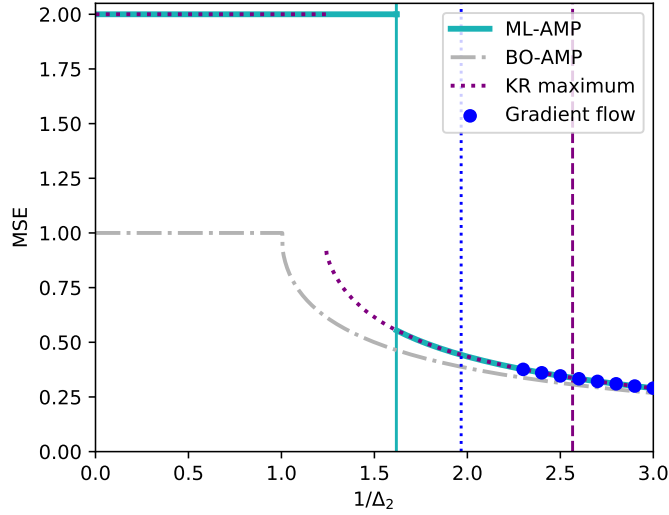}
		\caption{We show the mean-squared error (MSE) as achieved by
			the analyzed algorithms, for $p=3$, $\Delta_p = 1.0$ as a
			function of the signal-to-noise (snr) ratio  $1/\Delta_2$.
			The full cyan line corresponds to the error reached by the ML-AMP
			algorithm, it jumps discontinuously at $1/\Delta_2^{\rm ML-AMP}=1.62$.
			The blue points is the error reached by the
			gradient flow in time $t<1000$. The divergence of the convergence time
			is extrapolated to occur at $1/\Delta_2^{\rm GF} =1.97$, blue dotted vertical line.
			%
			%
			The purple dotted line represents the maximum having the largest $m$ of
			the complexity function $\Sigma(m)$,
			Eq.~\eqref{eq:complexity}. The vertical purple dashed line at
			$1/\Delta_2^{\rm triv} = 2.57$
			corresponds to the trivialization of the landscape, beyond
			which only local minima well correlated with the signal
			remain.
			We note that all these approaches agree on the value of the
			MSE. 
			For the sake of comparison we show (the dashed-dotted grey line) also
			the minimal-MSE achieved in the Bayes-optimal setting.
		}
		\label{fig:fixed_points_comparison}
	\end{center}
	\vskip -0.2in
\end{figure}

Approximate Message Passing (AMP) is a popular iterative algorithm
\cite{DMM09} with a key advantage of being analyzable 
		via a set of equations, called state evolution equations, that have
		been proved rigorously to follow the average evolution of the
		algorithm \cite{javanmard2013state}.
The maximum-likehood AMP (ML-AMP) algorithm studied in this paper is a generalization of AMP for
the pure spiked tensor model from
\cite{richard2014statistical} to the spiked matrix-tensor model. We will show that its fixed points
correspond to stationary points of the loss function \eqref{eq:Hamiltonian}. This should be contrasted with
the Bayes-optimal AMP (BO-AMO) that was
studied in \cite{sarao2018marvels} and aims to approximate the marginals of
the corresponding posterior probability distribution. The ML-AMP
instead aims to estimate the maximum-likelihood solution, $\hat{x}$. In
information theory the BO-AMP would correspond to the sum-product
algorithm, while the present one to the max-sum algorithm. In
statistical physics language the BO-AMP corresponds to temperature
one, while the present one to zero temperature. In the supporting
information we provide a schematic derivation of the ML-AMP as a
zero-temperature limit of the BO-AMP, using a scheme similar to \cite{LKZ17}.

The ML-AMP algorithm reads
\begin{align}
	\label{eq:0T-AMP_B}
	\begin{split}
		&B^t_i = \frac{\sqrt{(p-1)!}}{N^{(p-1)/2}}\sum_{k_2<\dots <k_p}\frac{T_{ik_2\dots k_p}}{\Delta_p}\hat{x}_{k_2}^t\dots \hat{x}_{k_p}^t
		+ \frac{1}{\sqrt{N}}\sum_{k}\frac{Y_{ik}}{\Delta_2}\hat{x}_k^t - \text{r}_t \hat{x}_i^{t-1}\,,
	\end{split}
	\\
	\label{eq:0T-AMP_x}
	 & \hat{x}_i^{t+1}=\frac{B^t_i}{\frac1{\sqrt{N}}||B^t||_2}\,,
	\\
	\label{eq:0T-AMP_sigma}
	 & \hat{\sigma}^{t+1}=\frac1{\frac1{\sqrt{N}}||B^t||_2}
\end{align}
with $||\cdots||^2_2$ the $\ell_2$-norm and $\text{r}_t$ the Onsager reaction term
\begin{equation}\label{eq:0T-AMP_onsager}
	\begin{split}
		\text{r}_t &= \frac1{\Delta_2}\frac1{N}\sum_k\hat\sigma_k^t 
		+\frac{p-1}{\Delta_p}\frac1{N}\sum_k\hat\sigma_k^t\left(\frac1{N}\sum_k\hat{x}_k^t\hat{x}_k^{t-1}\right)^{p-2}\,.
	\end{split}
\end{equation}

\subsection{ML-AMP \& Stationary Points of the Loss}\label{sec:AMP ridge}

Using an argument similar to Prop. 5.1 in
\cite{montanari2012graphical} we can show that a fixed points found by
ML-AMP corresponds to finding a stationary point of the loss Eq.~\eqref{eq:Hamiltonian} with a ridge regularizer.

\begin{property}\label{prop:0T-AMP_ridge}
	Given $(\hat{x}^*,\sigma^*)$ a fixed point of ML-AMP, then
	$\hat{x}^*$ satisfies the stationary condition of the loss.
\end{property}
\begin{proof}
	Let us denote $B^*$, $\text{r}^*$ the fixed point of
	Eqs.~\eqref{eq:0T-AMP_B} and \eqref{eq:0T-AMP_onsager}. From Eq.~\eqref{eq:0T-AMP_x} and Eq.~\eqref{eq:0T-AMP_B} we have
	\begin{equation}
		\begin{split}
			&\left(\frac1{\sqrt{N}}||B^*||_2 +\text{r}^*\right)x^* = \frac{1}{\sqrt{N}}\sum_{k}\frac{Y_{ik}}{\Delta_2}\hat{x}^*_i
			+ \frac{\sqrt{(p-1)!}}{N^{(p-1)/2}}\sum_{k_2<\dots <k_p}\frac{T_{ik_2\dots k_p}}{\Delta_p}\hat{x}_{k_2}^*\dots \hat{x}_{k_p}^*
		\end{split}  \label{eq:amp_rige}
	\end{equation}
	which is exactly solution of the derivative of Eq.~\eqref{eq:Hamiltonian} with respect to $x_i$ when the spherical constraint is enforced by a Lagrange multiplier $\mu$
	\begin{equation*}
		\begin{split}
			0 &= - \mu x_i + \frac{1}{\sqrt{N}}\sum_{k}\frac{Y_{ik}}{\Delta_2}x_i
			 +
			\frac{\sqrt{(p-1)!}}{N^{(p-1)/2}}\sum_{k_2<\dots
				<k_p}\frac{T_{ik_2\dots
						k_p}}{\Delta_p}x_{k_2}\dots x_{k_p} \, .
		\end{split}
	\end{equation*}
	Moreover ML-AMP by construction preserves the spherical constrain
	at every time iteration.
\end{proof}

\subsection{State Evolution}\label{sec:AMP SE}

The evolution of ML-AMP can be tracked through a set of equations called
state evolution (SE).
The state evolution can be characterized via an order parameter: $m^t
	= \frac1N\sum_i \hat{x}_i^tx_i^*$, the correlation of the ML-AMP-estimator
with the ground truth signal at time $t$.
According to the SE, as derived in the supporting
information, and proven for a general class of models in
\cite{javanmard2013state}, this parameter
evolves in the large $N$ limit as
\begin{align}
	\label{eq:0T-AMP_SE_m}
	 & m^{t+1} =\frac{\frac{m^t}{\Delta_2} + \frac{(m^t)^{p-1}}{\Delta_p}}{\sqrt{\frac1{\Delta_2}+\frac1{\Delta_p}+\left(\frac{m^t}{\Delta_2} + \frac{(m^t)^{p-1}}{\Delta_p}\right)^2}}\,,
\end{align}
and the mean square error correspondingly
\begin{equation}
	\text{MSE}^t = 2(1-m^{t}).
\end{equation}

Analysis of the simple scalar SE, Eq.~\eqref{eq:0T-AMP_SE_m},
allows to identify the error reached by the ML-AMP algorithm. We
first observe that $m=0$ is always a fixed point. For the performance
of ML-AMP is the stability of this fixed point that determines whether
the ML-AMP will be able to find a positive correlation with the signal or
not. Analyzing Eq.~\eqref{eq:0T-AMP_SE_m} we obtain that the $m=0$ is
a stable fixed point  for $\Delta_2 > \Delta_2^{\rm ML-AMP}$ where
\begin{equation}\label{eq:AMP stability m0}
	\Delta_2^{\rm ML-AMP}(\Delta_p) =
	\frac{-\Delta_p+\sqrt{\Delta_p^2+4\Delta_p}}2\,.
\end{equation}
Consequently for $\Delta_2 > \Delta_2^{\rm ML-AMP}$ the ML-AMP algorithm
converges to $m=0$, i.e. zero correlation with the signal.
The line $\Delta_2^{\rm ML-AMP}$  is the line plotted in
Fig.~\ref{fig:phase_diagram_simple}. For $p=3$ and $p=4$, we obtain that for $\Delta_2 < \Delta_2^{\rm ML-AMP}$
the ML-AMP algorithm converges to a positive  $m^*>0$ correlation with the
signal, depicted in Fig.~\ref{fig:fixed_points_comparison}. In
Fig.~\ref{fig:fixed_points_comparison} we also observe that this
correlation agrees 
with the position
of the maximum having largest value of $m$ in the complexity function
$\Sigma(m)$.
The trivialization of the landscape
occurs at $\Delta_2^{\rm triv} < \Delta_2^{\rm ML-AMP}$, thus showing
that for $\Delta_2^{\rm triv} < \Delta < \Delta_2^{\rm ML-AMP}$ the ML-AMP algorithm is able to ignore a good portion of the spurious local minima
and to converge to the local minima best correlated with the signal.

In Fig.~\ref{fig:fixed_points_comparison} we also compared to the MSE
obtained by the Bayes-optimal AMP that provably minimizes the MSE in
the case depicted in the figure \cite{sarao2018marvels}. We see that the gap between the
Bayes-optimal error and the one reached by the loss minimization
approaches goes rapidly to zero as $\Delta_2$ decreases.


\section{Discussion}

We analyzed the behavior of two algorithms for
optimizing a rough high-dimensional  loss landscape of the spiked
matrix-tensor model.
We used the Kac-Rice formula to count the average number of minima of the
loss function having a given correlation with the signal. Analyzing
the resulting formula we defined and located where the energy landscape
becomes trivial in the sence that spurious local minima disappear.
We analyzed the performance of gradient flow via integro-differential
state-evolution-like equations.
We delimited a region of parameters
for which the gradient flow is able to avoid the spurious minima and obtain a
good correlation with the signal in time linear in the input size.
We also analyzed the maximum-likelihood AMP algorithm, located the
region of parameters in which this algorithm works, which is larger
than the
region for which the gradient flow works.
%
%
The relation between
existence or absence of spurious local minima in the
loss landscapes of a generic optimization problems and the actual performance of optimization algorithm
is yet to be understood. Our analysis of the spiked matrix-tensor model
brings a case-study where we were able to specify this relation quantitatively.
%

\section*{Acknowledgments}

We thank G. Ben Arous, G. Biroli, C. Cammarota, G. Folena, and V. Ros for precious discussions.
We acknowledge funding from the ERC under the European
Union’s Horizon 2020 Research and Innovation Programme Grant
Agreement 714608-SMiLe and 307087-SPARCS; from the French National
Research Agency (ANR) grant PAIL; and from "Investissements d’Avenir"
LabEx PALM (ANR-10-LABX-0039-PALM) (SaMURai
and StatPhysDisSys). The manuscript was finalized while some of the authors were visiting KITP thus they acknowledge partial support by the National Science Foundation under Grant No. PHY-1748958.





\appendix

\section{Kac-Rice formula}

\begin{figure*}[h]
	\vskip 0.2in
	\begin{center}
	\centering
	\subfigure[$\Delta_2 = 2$]{
	    \includegraphics[width=.31\columnwidth]{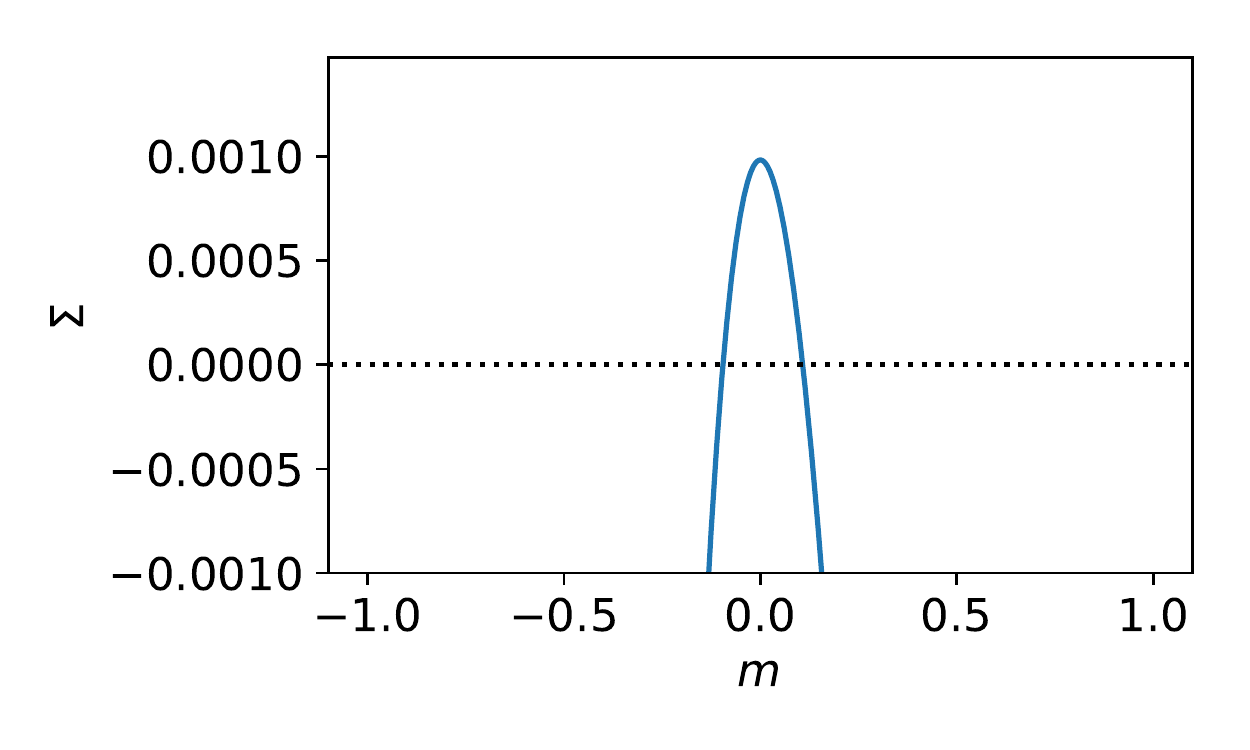}
	}
	\subfigure[$\Delta_2 = 2/3$]{
	    \includegraphics[width=.31\columnwidth]{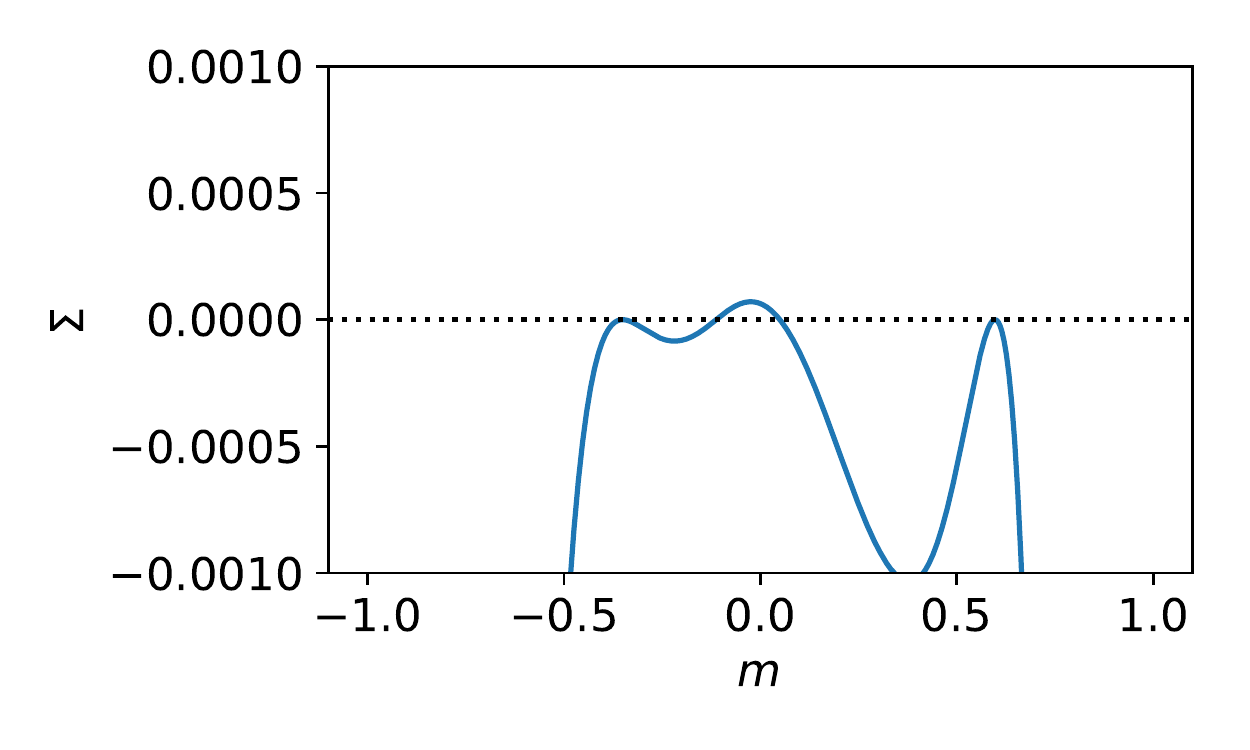}
	}
	\subfigure[$\Delta_2 = 2/5$]{
	    \includegraphics[width=.31\columnwidth]{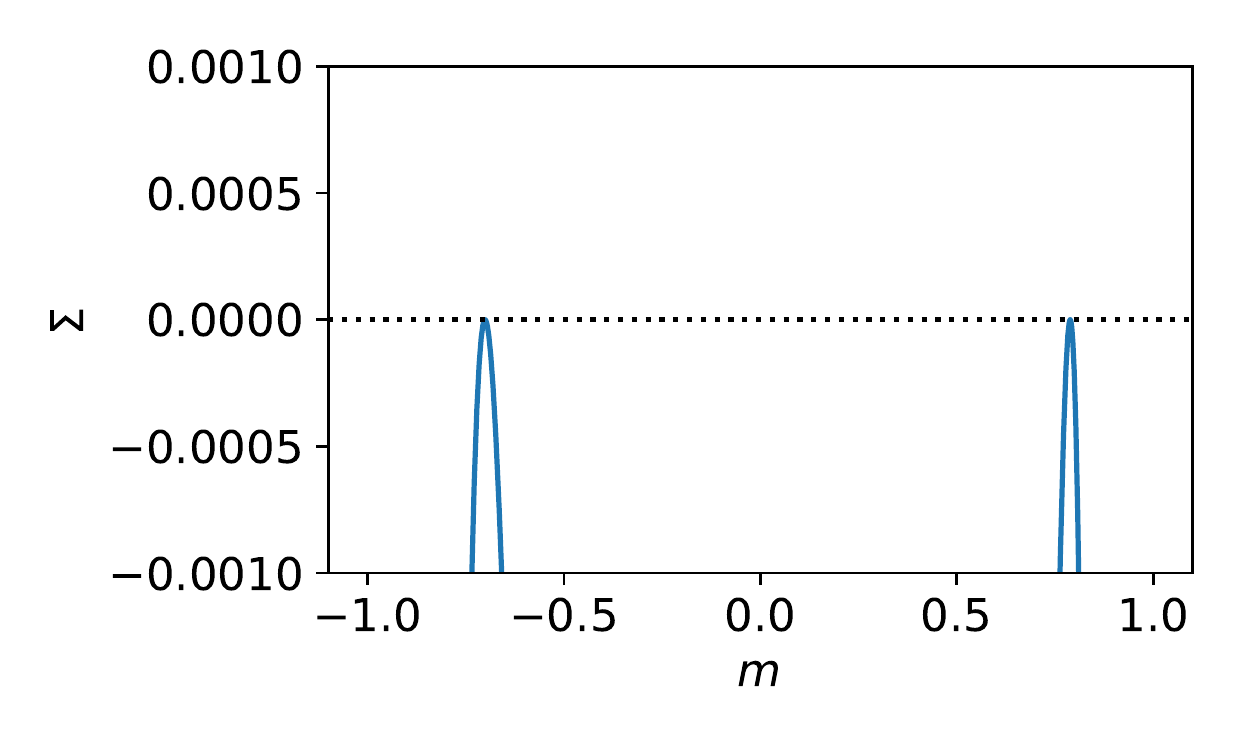}
	}
	\caption{Analogously to Fig.~\ref{fig:complexity}, the figures
          show the complexity, Eq.~\eqref{eq:complexity}, as a
          function of the correlation with the signal for different values of parameter $\Delta_2$ at fixed $\Delta_p=4.0$ in the case $p=3$.}
	\label{fig:complexity_p3}
	\end{center}
\end{figure*}

\subsection{$p$-odd cases}\label{app:p_odd}

\begin{figure}[ht]
	\vskip 0.2in
	\begin{center}
	\centering
	    \includegraphics[width=.45\columnwidth]{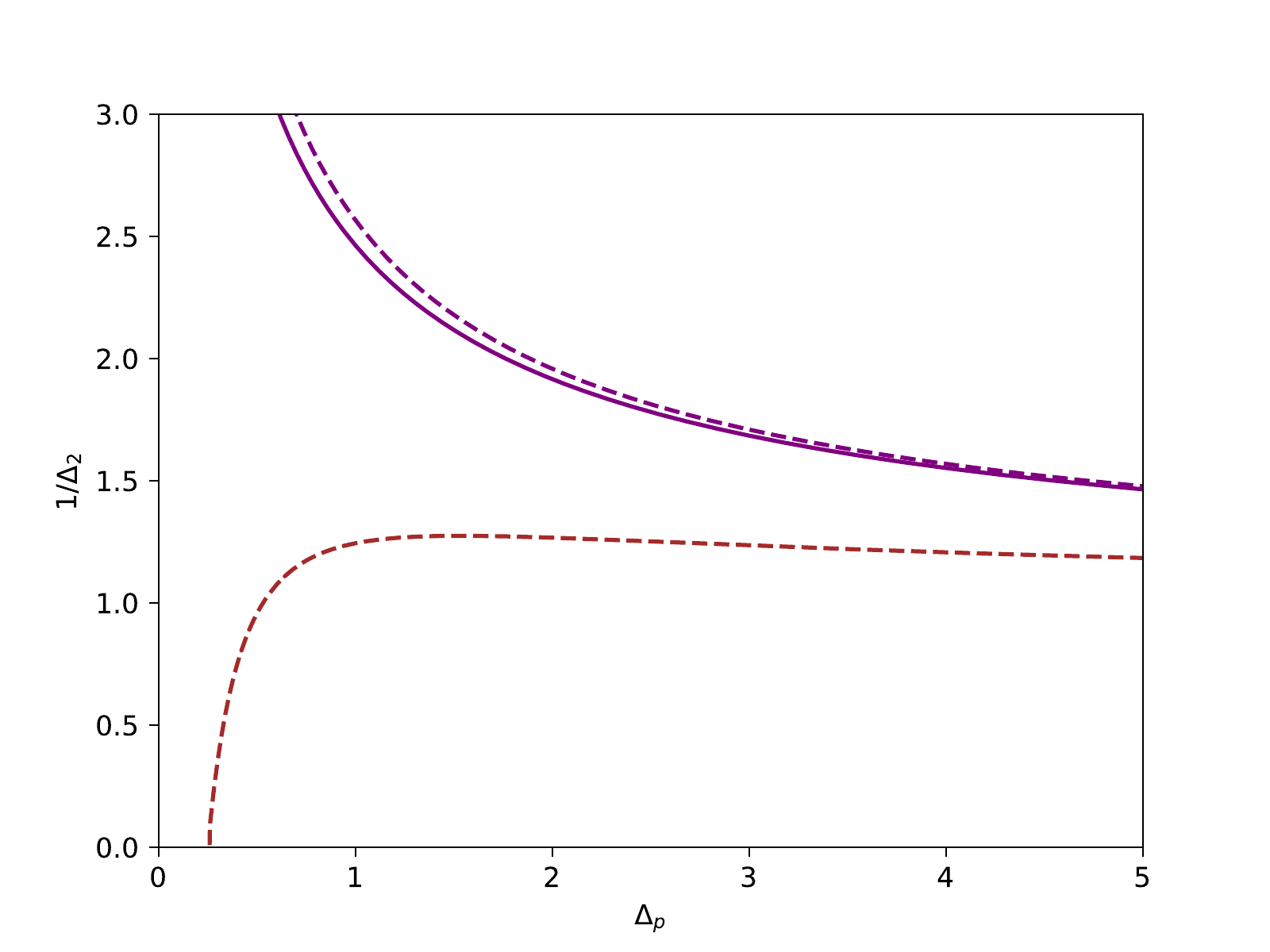}
	\caption{The thresholds representing the trivialization of the
          landscape (purple) and the point where the support of
          $\Sigma(m)\ge 0$ become disconnected (brown) for tensors of
          order $p=3$. We compare the two definitions of the
          trivialization threshold described in Sec.~\ref{app:p_odd}:
          the solid line considers just the positivity of the
          complexity Eq.~\eqref{eq:complexity} at $m=0$, the dashed
          line considers the whole non-informative band.}
	\label{fig:phase_diagram_p_odd}
	\end{center}
\end{figure}

In the cases in which the order of the tensor $p$ is odd we encounter
an interesting phenomenon due to the different symmetries of the two
types of observation. The matrix is symmetric by inverting the sign of
the signal, $\hat{x} \mapsto -\hat{x}$, while the tensor is not
symmetric for odd $p$. This creates an asymmetry in the complexity,
Fig.~\ref{fig:complexity_p3} (to be compared with
Fig.~\ref{fig:complexity}) and causes a shift toward lower
correlations of the band characterizing the non-informative
minima. Therefor observing when the complexity at $m=0$ becomes
negative does not guarantee that the non-informative minima
disappeared. To do so, one must check that the whole non-informative
band disappears. This should be contrasted with the case of even $p$
where a maximum of the complexity $\Sigma(m)$ is always at $m=0$. 
These two definitions of the threshold have little, but not
negligible, difference, see
Fig.~\ref{fig:phase_diagram_p_odd}. Observe that as $\Delta_p$
increases the peak of the complexity decreases, since the loss
Eq.~\eqref{eq:Hamiltonian} tends to the simple matrix-factorization
problem where the landscape is characterized by two isolated
minima. This implies that the two definitions become indistinguishable
for large $\Delta_p$. In the main text we use the definition taking
into account the maximum (even when it is not strictly at $m=0$) because gives a more accurate characterization of the trivialization threshold.

\section{Gradient Flow}

\subsection{Dependence on the initial conditions}

\begin{figure}[ht]
	\vskip 0.2in
	\begin{center}
	\centering
	    \includegraphics[width=.45\columnwidth]{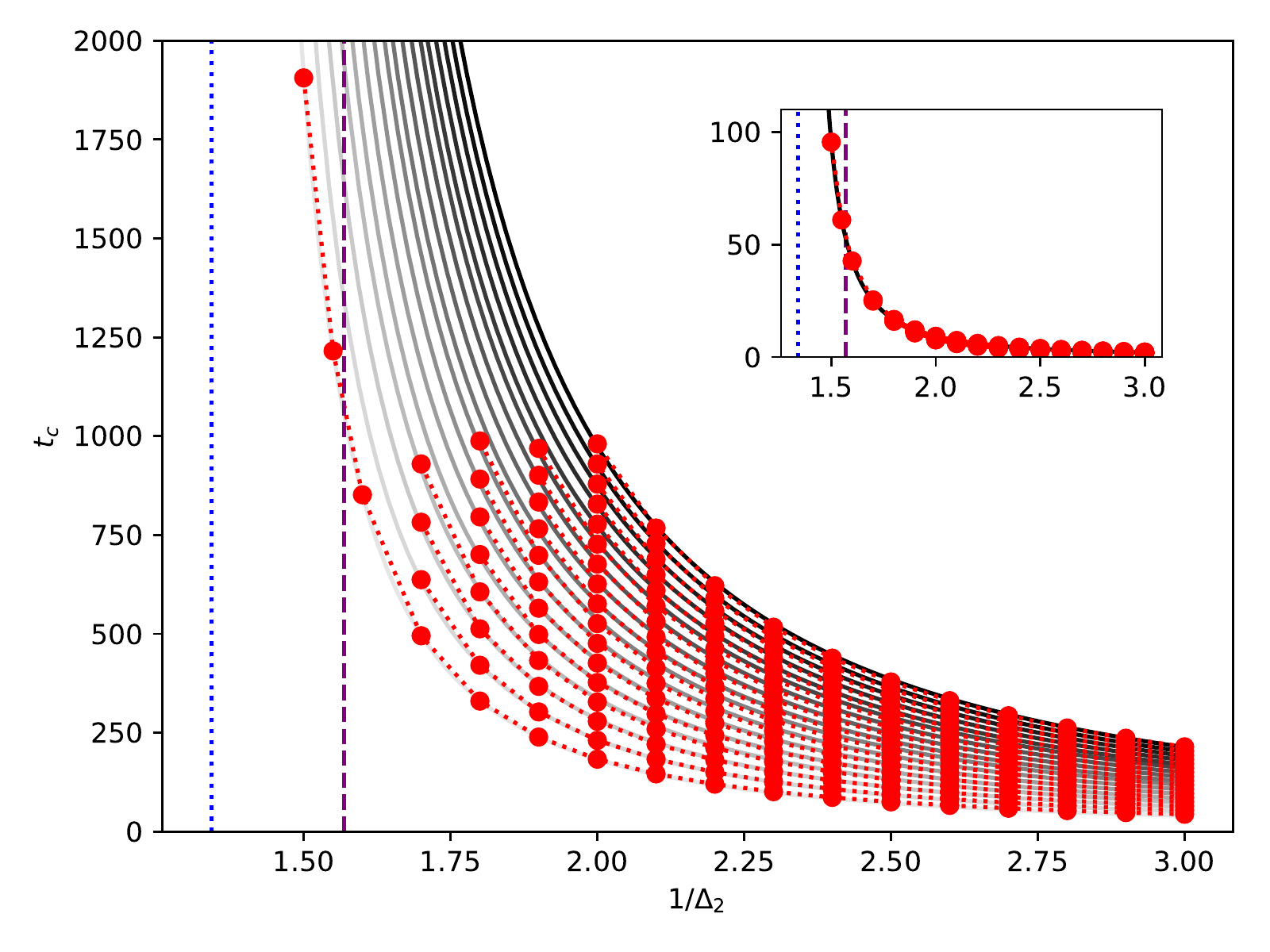}
	\caption{The time corresponding to convergence close to the
          signal is shown for $\Delta_p=4.0$ in the case
          $p=3$. Different shades of grey correspond to different initial conditions, from $\Cmag(0) = 10^{-10}$ (light grey) to $\Cmag(0) = 10^{-42}$ (dark grey). The different initializations collapse to a single line when the time is rescaled by $a^{\log\Cmag(0)}$ with $a=1.3$, see inset. In the figure we fit only the case $\Cmag(0) = 10^{-10}$ with a power law and use the same parameters for all the other fits with a vertical translation. The divergence point extrapolated is $1/\Delta_2^{\text{GF}} = 1.35$ and is represented by the vertical dotted line, while the dashed line identifies the landscape trivialization predicted with the Kac-Rice formula,  $1/\Delta_2^{\text{triv}} = 1.57$.
	}
	\label{fig:GD_limits_multi}
	\end{center}
\end{figure}

The dynamics of the gradient flow shows a dependence on the initial
conditions, because formally zero correlation is a (unstable) fixed
point of the GF state evolution. In practice we observe for both GF
and ML-AMP that instability of the fixed point is sufficient for good
performance of the algorithm. However, this makes the definition of
the convergence time depend of the initial condition. 

We observed from our numerical solution of the GF state evolution
equations that the initial condition add a factor $a^{\log\Cmag(0)}$ to the convergence times. Thus by fitting this term and rescaling the convergence times, the different initializations collapse into a single curve, see inset of Fig.~\ref{fig:GD_limits_multi}. Finally, the collapsed points were used to extrapolate the critical line as shown in the main text, Fig.~\ref{fig:GD_limits}.

\section{AMP}\label{sec:app_AMP}


\subsection{From AMP to ML-AMP}\label{sec:app_zero_T}

In this section we consider the spiked-tensor model in a Bayesian
way. We show how the Bayes-optimal AMP leads to the Maximum Likelihood
AMP using a temperature-like parameter $T$. We will introduce the
algorithm AMP for a generic $T$, and show that as $T\rightarrow0$ we
recover ML-AMP as presented in the main text. The probability
distribution we consider is 
\begin{equation}\label{eq:posterior general}
    \begin{split}
        P(X&|Y,T) \propto e^{-\mu\,\norm{x}^2}\prod_{i< j}e^{-\frac1{2T\Delta_2}\left(Y_{ij}-\frac{x_ix_j}{\sqrt{N}}\right)^2}
        \prod_{ i_1<\dots<i_p}e^{-\frac1{2T\Delta_p}\left(T_{i_1\dots i_p}-\frac{\sqrt{(p-1)!}}{N^{(p-1)/2}}x_{i_1}\dots x_{i_p}\right)^2}.
    \end{split}
\end{equation}
The scheme for deriving AMP estimating marginals of such a probability
distribution can be found in \cite{LKZ17,sarao2018marvels} and
consist in making a Gaussian assumption on the distribution of the
messages in the belief propagation (BP) algorithm and neglecting the
node-dependence in the messages. A final consideration to be used in
order to derive the algorithm is that the spherical constrain can be
imposed by setting $\frac1N\sum_i(\hat{x_i}^2+\sigma_i) = 1$ at every
iteration. The resulting AMP algorithm will iterate on the following equations:
\begin{align}
    \label{eq:AMP_B}
	\begin{split}
		&B^t_i = \frac{\sqrt{(p-1)!}}{N^{(p-1)/2}}\sum_{k_2<\dots <k_p}\frac{T_{ik_2\dots k_p}}{T\Delta_p}\hat{x}_{k_2}^t\dots \hat{x}_{k_p}^t
        + \frac{1}{\sqrt{N}}\sum_{k}\frac{Y_{ik}}{T\Delta_2}\hat{x}_k^t - \text{r}_t \hat{x}_i^{t-1}
	\end{split}
	\\
    \label{eq:AMP_x}
    &\hat{x}_i^{t+1}=2\frac{B^t_i}{1+\sqrt{1+\frac4N ||B^t||^2_2}}\,,
    \\
    \label{eq:AMP_sigma}
    &\sigma^{t+1}=\frac2{1+\sqrt{1+\frac4N ||B^t||^2_2}}\,.
\end{align}
with $||\cdots||^2_2$ the $\ell_2$-norm and $\text{r}_t$ the Onsager reaction term
\begin{equation}\label{eq:AMP_onsager}
	\begin{split}
    	\text{r}_t &= \frac1{\Delta_2T^2}\frac1{N}\sum_k\sigma_k^t
    	+\frac{p-1}{\Delta_pT^2}\frac1{N}\sum_k\sigma_k^t\left(\frac1{N}\sum_k\hat{x}_k^t\hat{x}_k^{t-1}\right)^{p-2}\,.
    \end{split}
\end{equation}
%

In the limit $T\rightarrow0$ AMP defined by Eqs.~(\ref{eq:AMP_B}-\ref{eq:AMP_onsager}) is equivalent to ML-AMP, Eqs.~(\ref{eq:0T-AMP_B}-\ref{eq:0T-AMP_onsager}).
To see this we	define the rescaled variables $\hat{\sigma}^t \doteq \sigma^t/T$, $\tilde{B}^t \doteq T\, B^t$ and $\tilde{\text{r}}_t \doteq T\text{r}_t$. Taking the limit $T\rightarrow0$ the expression for $\hat{x}^{t+1}_i$ Eq.~\eqref{eq:AMP_x} and the expression for $\hat{\sigma}^{t+1}_i$ Eq.~\eqref{eq:AMP_sigma} simplify as Eq.~\eqref{eq:0T-AMP_x} and as Eq.~\eqref{eq:0T-AMP_x} respectively. Dropping the tildes we obtain ML-AMP as presented in the main text. 

\subsection{State evolution}

The generic $T$ version of AMP has a slightly more complicated SE that
depends of two order parameters: the already introduced $m^t =
\frac1N\sum_i \hat{x}_i^tx_i^*$ and $q^t = \frac1N\sum_i ( \hat{x}_i^t
)^2$ the self-overlap of the estimator.
The SE equations are:
\begin{align}
	\label{eq:AMP_SE_m}
    & m^{t+1} = 2 \frac{z^t(T)}{1+\sqrt{1+4y^t(T)}}\,,
    \\
	\label{eq:AMP_SE_q}
    & q^{t+1} = 4 \frac{y^t(T)}{\left(1+\sqrt{1+4y^t(T)}\right)^2}
\end{align}
and
\begin{equation}
    \text{MSE}^t = 1-2m^{t}+q^{t}\,,
\end{equation}
with $y^t(T) = \left(z^t(T)\right)^2+\left(\frac1{T^2}\frac{q^t}{\Delta_2}+\frac1{T^2}\frac{(q^t)^{p-1}}{\Delta_p}\right)$ and $z^t(T) = \frac1{T}\frac{m^t}{\Delta_2}+\frac1{T}\frac{(m^t)^{p-1}}{\Delta_p}$. 


Given $\frac1N||\hat{x}^0||^2_2 \neq 0$, in the limit $T\rightarrow0$ AMP SE Eqs.~(\ref{eq:AMP_SE_m}-\ref{eq:AMP_SE_q}) simplify, to a single equation corresponding to ML-AMP SE Eq.~\eqref{eq:AMP SE}.
This is seen by taking the limit for Eq.~\eqref{eq:AMP_SE_q} which gives $q^t = 1\;\forall t>0$, implying $\text{MSE}^t = 2(1-m^t)$. Then, using the result for $q^t$, we show that Eq.~\eqref{eq:AMP_SE_m} tends to Eq.~\eqref{eq:0T-AMP_SE_m}.

\subsection{Derivation of spinodals}\label{sec:app_spinodals}

\begin{figure*}[ht]
	\vskip 0.2in
	\begin{center}%
	\centering
	\subfigure[$p=3$]{
	    \includegraphics[width=.45\columnwidth]{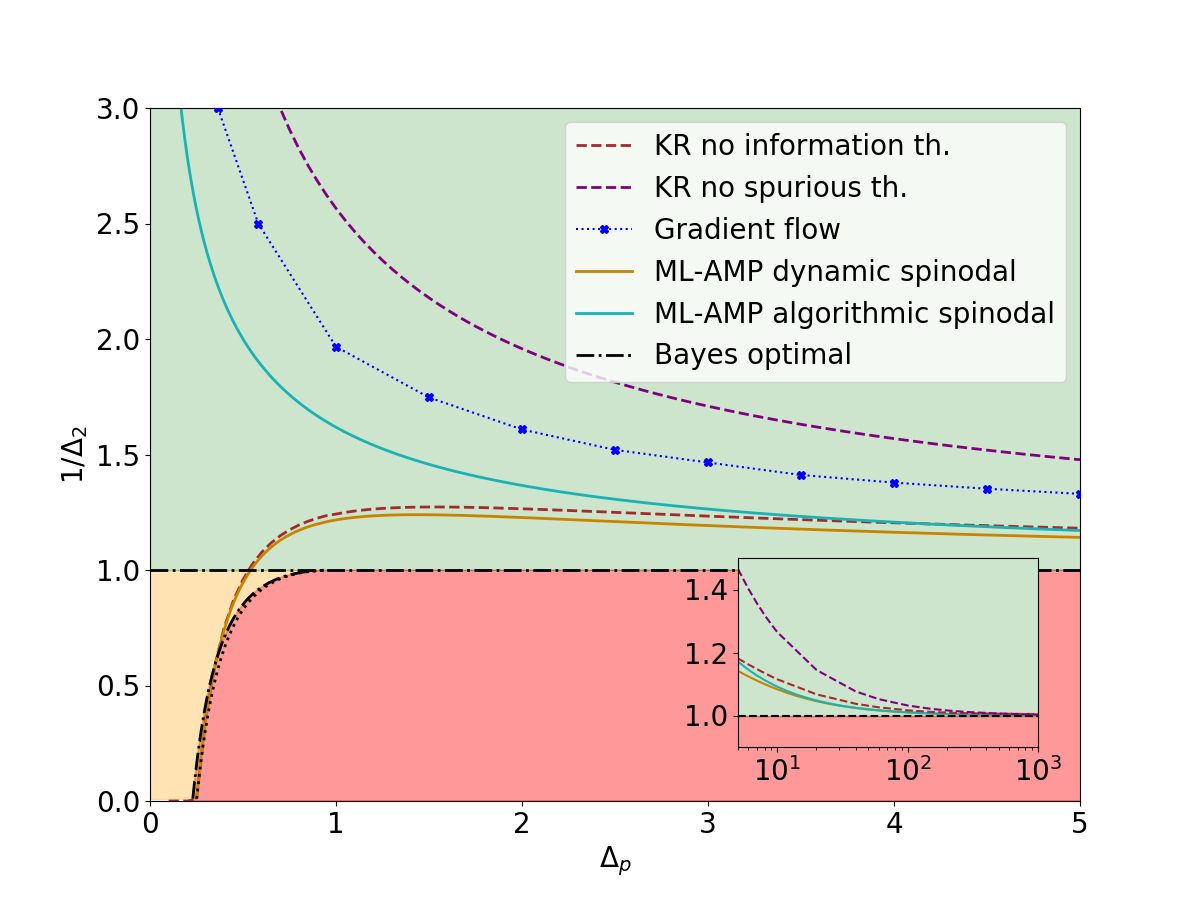}
	}
	\subfigure[$p=4$]{
	    \includegraphics[width=.45\columnwidth]{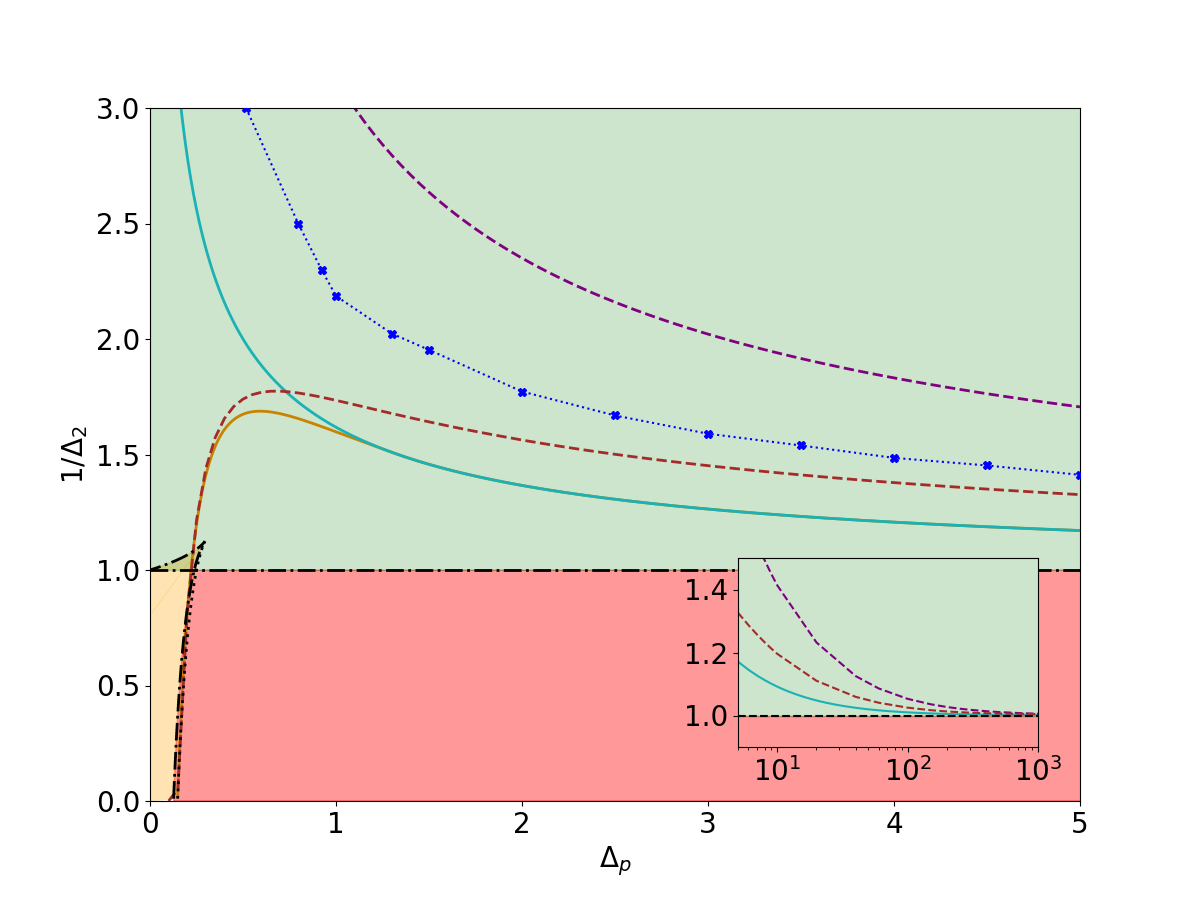}
	}
	\caption{The phase diagram already describe in
          Fig.~\ref{fig:phase_diagram_simple} with two additional
          lines. The dashed brown line is the limit predicted by
          Kac-Rice formula where the support of the $\Sigma(m) \ge 0$
          becomes disconnected (above the line). The full orange line
          is related to the ML-AMP algorithm, is called
          \textit{dynamical spinodal}, below it the algorithm
          converenges to $m=0$ even if initialized in the
          solution. 
		In the insets we show the large $\Delta_p$ behaviour of the thresholds, where we can observe that the lines merge at infinity.
}
	\label{fig:phase_diagram_full}
	\end{center}
	\vskip -0.2in
\end{figure*}

From SE Eq.~\eqref{eq:0T-AMP_SE_m} we can obtain analytical equations
for the \textit{spinodals}, the threshold of stability of the different ML-AMP fixed points. We have $\hat{x}^{t+1} = f_{SE}(z^t)$ with
\begin{equation}\label{eq:AMP SE}
    f_{SE} (z) = \frac{z}{\sqrt{z^2+\gamma}}, 
\end{equation}
with $\gamma = 1/\Delta_2 + 1/\Delta_p$ and $ z = m/\Delta_2 + m^{p-1}/\Delta_p$. Observe that: $f_{SE}'(z) = \frac{\gamma}{(z^2+\gamma)^{\frac32}}$.
We can now define either $\Delta_p\equiv\Delta_p(z;\Delta_2,\gamma) = \frac{f_{SE}(z)^{p-1}}{z-\frac{f_{SE}(z)}{\Delta_2}}$ or $\Delta_2\equiv\Delta_2(z;\Delta_p,\gamma) = \frac{f_{SE}(z)}{z-\frac{f_{SE}(z)^{p-1}}{\Delta_p}}$.

As remarked in \cite{sarao2018marvels}, the spinodals are given by the following condition:
\begin{equation}\label{eq:AMP spinodal implicit}
	\begin{split}
	    0 &= \frac{d\log\Delta_2}{dm}\propto \frac{d\log\Delta_2}{dz} =
	    \frac{z\left[(p-2)\gamma\left(\frac{z}{\sqrt{z^2+\gamma}}\right)^{p-1}-z^3\Delta_p\right]}{z(z^2+\gamma)\left[\Delta_pz^2-z\left(\frac{z}{\sqrt{z^2+\gamma}}\right)^{p-1}\right]}.
	\end{split}
\end{equation}

\begin{figure}
	\begin{center}
	    \includegraphics[width=.45\columnwidth]{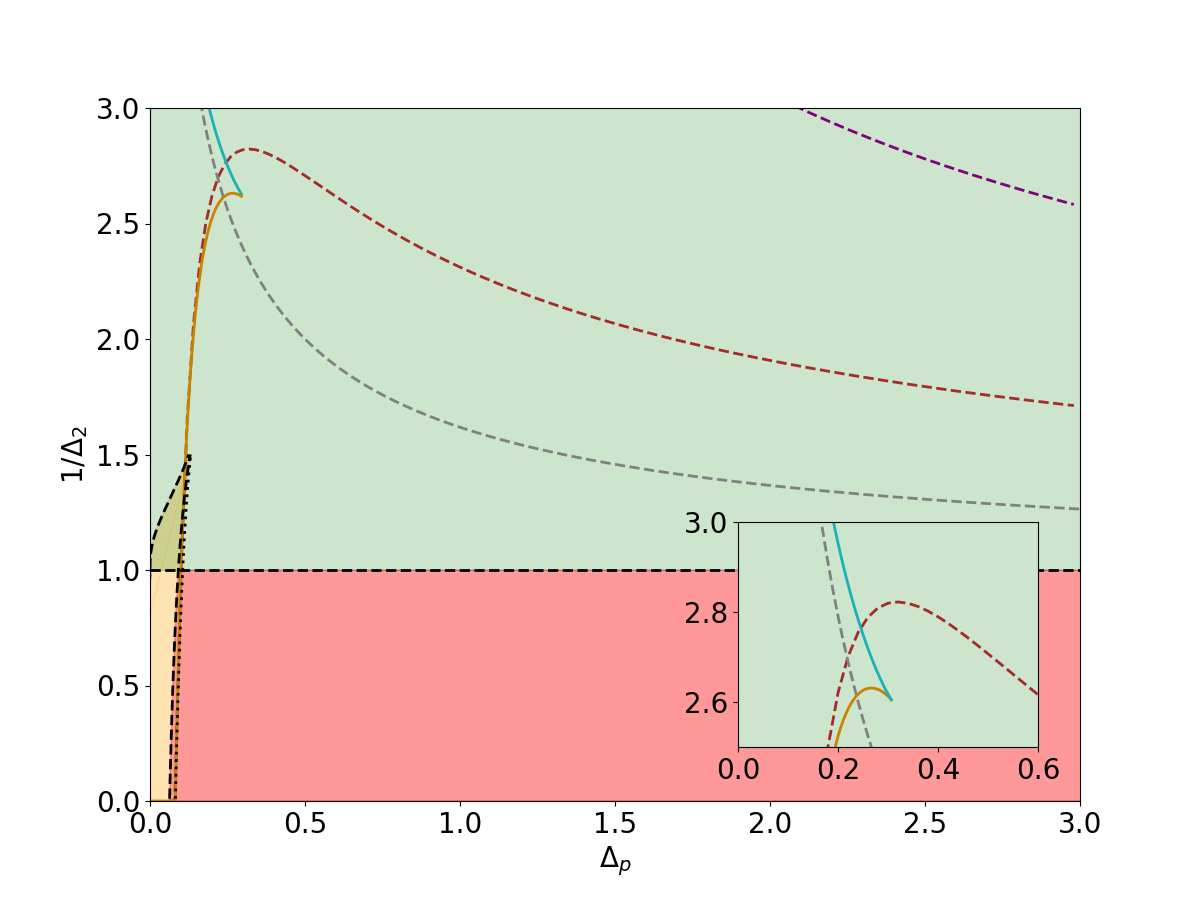}
	\caption{Phase diagram as shown in
          Fig.~\ref{fig:phase_diagram_full} for the case $p=6$. The
          difference between $p=3,4$ and $p>4$ is that a new phase
          appears, called hybrid hard phase, where two fixed points of
          ML-AMP aligned with the signal are present and the
          convergence to one or the other depends on the
          initialization. The region is highlighted in the inset. In
          the phase diagram the grey dashed line represent the
          threshold above which the non-informative fixed point becomes unstable.}
	\label{fig:phase_diagram_p6}
	\end{center}
\end{figure}

A trivial solution is given by $z\rightarrow0$ corresponding to stability of the non-informative solution $m=0$, and gives the \textit{algorithmic spinodal} for the cases $p\in\{3,4\}$. This solution and has a very simple equation for every $p$: $\Delta_2 =1/\sqrt\gamma$ giving Eq.~\eqref{eq:AMP stability m0}, already presented in the main text.
An interesting implication of Eq.~\eqref{eq:AMP stability m0} is that it is independent from the value of $p$, it is in some sense \textit{universal} among the $2+p$-models.

The expression for the stability of the informative solution, \textit{dynamical spinodal}, is less straightforward, but analytical progresses can be done in the cases $p=3$ and $p=6$ (using Cardano formula) and in the case $p=4$ for which it is equivalent to a second order polynomial
\begin{equation}
    z^2+\gamma = \left(\frac{(p-2)\gamma}{\Delta_p}\right)^{\frac2{(p-1)}} = \left(\frac{2\gamma}{\Delta_4}\right)^{\frac23}\,,
\end{equation}
that admits a single solution in $\mathbb{R}^+$:
\begin{equation}
    z = \sqrt{\left(\frac{2\gamma}{\Delta_4}\right)^{\frac23}-\gamma}\,.
\end{equation}

An important point in the phase diagram is where the algorithmic and
dynamical spinodals meet, this is called the \textit{tricritical point}. Its value is obtained for different $p$, numerically (for $p>4$) and analytically (for $p=4$), and is reported in Table~\ref{table:tricritical points}.
The case $p=3$ does not show any tricritical point for any finite $\Delta_p$, the two lines eventually meet at $\Delta_p = \infty$ when the spiked matrix problem is recovered.
\begin{table}
	\centering
	\begin{tabular}{c|c|c}
	 $p$ & $\Delta_2$ & $\Delta_p$ \\ \hline
	 4 & $\frac23 \simeq 0.667$ & $\frac43 \simeq 1.333$ \\ \hline
	 5 & 0.470 & 0.451 \\ \hline
	 6 & 0.384 & 0.305 \\ \hline
	 7 & 0.322 & 0.220 \\ \hline
	 8 & 0.279 & 0.172 \\ \hline
	 9 & 0.246 & 0.147 \\ \hline
	 10 & 0.220 & 0.121 
	\end{tabular}
	\caption{Table of the values of tricritical points for $p\ge4$.}
	\label{table:tricritical points}
\end{table}

For the cases $p>4$ we observe additionally the zero temperature
analog of what is called \textit{hybrid phase} in AMP in Bayes-optimal
regime \cite{Typology18}. The hybrid phase is illustrated in
Fig.~\ref{fig:phase_diagram_p6}. This phase is defined as a region
where the ML-AMP algorithm initialized at random converges to a
solution with positive correlation but that is less correlated then
the solution achievable starting from the solution. In these cases
Eq.~\eqref{eq:AMP stability m0} does not correspond to the algorithmic
spinodal but it is just the stability of the non-informative solution.

\end{document}